\documentclass[journal]{IEEEtran}
\usepackage[utf8]{inputenc} 
\usepackage[T1]{fontenc}    
\usepackage{hyperref}       
\usepackage{url}            
\usepackage{booktabs}       
\usepackage{amsfonts}       
\usepackage{nicefrac}       
\usepackage{microtype}      
\usepackage{graphicx}
\usepackage{amssymb,amsmath,amsthm,array,amsfonts}
\usepackage{xcolor}
\usepackage{booktabs}
\usepackage{underoverlap}
\usepackage{subfigure}
\usepackage{multirow} 
\usepackage[rightcaption]{sidecap}
\sidecaptionvpos{figure}{c}
\usepackage{balance}
\usepackage{algorithm}
\usepackage[noend]{algorithmic}

\def\x{{\mathbf x}}

\def\z{{\mathbf z}}
\def\c{{\mathbf c}}
\def\v{{\mathbf v}}
\def\u{{\mathbf u}}

\def\X{{\mathbf X}}

\def\W{{\mathbf W}}

\def\Q{{\mathbf Q}}

\def\A{{\mathbf A}}
\def\L{{\mathbf L}}
\def\D{{\mathbf D}}
\def\S{{\mathbf S}}
\def\I{{\mathbf I}}

\newtheorem{theorem}{Theorem}
\newcommand{\maxcut}{\texttt{\small{MAXCUT}}}
\renewcommand{\t}[1]{\tiny{#1}}

\newcommand{\argmin}{\operatornamewithlimits{arg\ min}}

\title{Hierarchical Representation Learning in Graph Neural Networks with Node Decimation Pooling}

\author{Filippo Maria Bianchi$^{*}$, %
        Daniele Grattarola~\IEEEmembership{Student,~IEEE},
        Lorenzo Livi~\IEEEmembership{Member,~IEEE},
        Cesare Alippi~\IEEEmembership{Fellow,~IEEE}
\thanks{*filippo.m.bianchi@uit.no}
\thanks{F. M. Bianchi is with the Dept. of Mathematics and Statistics, UiT the Arctic University of Norway and with NORCE, Norwegian Research Centre}
\thanks{D. Grattarola is with the Faculty of Informatics, Universit\`a della Svizzera italiana, Switzerland}
\thanks{L. Livi is with Dept.s. of Computer Science and Mathematics, University of Manitoba, Canada, and Dept. of Computer Science, University of Exeter, United Kingdom}
\thanks{C. Alippi is with Faculty of Informatics, Universit\`a della Svizzera italiana, Switzerland, and Dept. of Electronics, Information, and Bioengineering, Politecnico di Milano, Italy}
}

\begin{document}

\maketitle

\begin{abstract}
In graph neural networks (GNNs), pooling operators compute local summaries of input graphs to capture their global properties, and they are fundamental for building deep GNNs that learn hierarchical representations. 
In this work, we propose the Node Decimation Pooling (NDP), a pooling operator for GNNs that generates coarser graphs while preserving the overall graph topology. 
During training, the GNN learns new node representations and fits them to a pyramid of coarsened graphs, which is computed offline in a pre-processing stage.

NDP consists of three steps. 
First, a node decimation procedure selects the nodes belonging to one side of the partition identified by a spectral algorithm that approximates the \maxcut{} solution.
Afterwards, the selected nodes are connected with Kron reduction to form the coarsened graph.
Finally, since the resulting graph is very dense, we apply a sparsification procedure that prunes the adjacency matrix of the coarsened graph to reduce the computational cost in the GNN.
Notably, we show that it is possible to remove many edges without significantly altering the graph structure.

Experimental results show that NDP is more efficient compared to state-of-the-art graph pooling operators while reaching, at the same time, competitive performance on a significant variety of graph classification tasks. 
\end{abstract}

\section{Introduction}

Generating hierarchical representations across the layers of a neural network is key to deep learning methods. This hierarchical representation is usually achieved through pooling operations, which progressively reduce the dimensionality of the inputs encouraging the network to learn high-level data descriptors.
Graph Neural Networks (GNNs) are machine learning models that learn abstract representations of graph-structured data to solve a large variety of inference tasks~\cite{7974879, zhang2019depth, bai2019learning, bai2019deep, zhang2018end}.
Differently from neural networks that process vectors, images, or sequences, the graphs processed by GNNs have an arbitrary topology. 
As a consequence, standard pooling operations that leverage on the regular structure of the data and physical locality principles cannot be immediately applied to GNNs. 

Graph pooling aggregates vertex features while reducing, at the same time, the underlying structure in order to maintain a meaningful connectivity in the coarsened graph. 
By alternating graph pooling and message-passing (MP) operations~\cite{gilmer2017neural}, a GNN can gradually distill global properties from the graph, which are then used in tasks such as graph classification.

In this paper, we propose \textit{Node Decimation Pooling} (NDP), a pooling operator for GNNs. 
NDP is based on node decimation, a procedure developed in the field of graph signal processing for the design of multi-scale graph filters~\cite{tremblay2018design}.
In particular, we build upon the \textit{multi-resolution} framework~\cite{shuman2016multiscale} that consists of removing some nodes from a graph and then building a coarsened graph from the remaining ones.
The NDP procedure that we propose pre-computes off-line (\textit{i.e.}, before training) a \textit{pyramid} of coarsened graphs, which are then used as support for the node representations computed at different levels of the GNN architecture.

The contributions of our work are the following.

\begin{enumerate}
    \item We introduce the NDP operator that allows to implement deep GNNs that have a low complexity (in terms of execution time and memory requirements) and achieve high accuracy on several downstream tasks.
    \item We propose a simple and efficient spectral algorithm that partitions the graph nodes in two sets by maximizing a \maxcut{} objective. Such a partition is exploited to select the nodes to be discarded when coarsening the graph.
    \item We propose a graph sparsification procedure that reduces the computational cost of MP operations applied after pooling and has a small impact on the representations learned by the GNN. 
    In particular, we show both analytically and empirically that many edges can be removed without significantly altering the graph structure.
\end{enumerate}

When compared to other methods for graph pooling, NDP performs significantly better than other techniques that pre-compute the topology of the coarsened graphs, while it achieves a comparable performance with respect to state-of-the-art \textit{feature-based} pooling methods.
The latter, learn both the topology and the features of the coarsened graphs end-to-end via gradient descent, at the cost of a larger model complexity and higher training time.
The efficiency of NDP brings a significant advantage when GNNs are deployed in real-world scenarios subject to computational constraints, like in embedded devices and sensor networks.

The paper is organized as follows: in Sect.~\ref{sec:preliminaries}, we formalize the problem and introduce the nomenclature; in Sect.~\ref{sec:method}, we present the proposed method; Sect.~\ref{sec:analysis} provides formal analyses and implementation details; related works are discussed in Sect.~\ref{sec:related_work}, and Sect.~\ref{sec:experiments} reports experimental results.
Further results and analyses are deferred to the supplementary material.

\section{Preliminaries}
\label{sec:preliminaries}
Let $G = \{\mathcal{V}, \mathcal{E} \}$ be a graph with node set $\mathcal{V}$, $|\mathcal{V}| = N$, and edge set $\mathcal{E}$ described by a symmetric adjacency matrix $\A \in \mathbb{R}^{N \times N}$. 
Define as \textit{graph signal} $\X \in \mathbb{R}^{N \times F}$ the matrix containing the features of the nodes in the graph (the $i$-th row of $\X$ corresponds with the features $\x_i \in \mathbb{R}^F$ of the $i$-th node).
For simplicity, we will only consider undirected graphs without edge annotations. 

Let $\L = \D - \A$ be the Laplacian of the graph, where $\D$ is a diagonal degree matrix s.t.\ $d_{ii}$ is the degree of node $i$.
We also define the symmetric Laplacian as $\L_s = \I - \D^{-1/2} \A \D^{-1/2}$.
The Laplacian characterizes the dynamics of a diffusion process on the graph and plays a fundamental role in the proposed graph reduction procedure. 
We note that in the presence of directed edges it is still possible to obtain a symmetric and positive-semidefinite Laplacian~\cite{chung2005laplacians, sandryhaila2013discrete} for which the derivations presented in this paper hold.

We consider a GNN composed of a stack of MP layers, each one followed by a graph pooling operation.
The $(l)$-th pooling operation reduces $N_{l}$ nodes to $N_{l+i} < N_l$, producing a pooled version of the node features $\X^{(l+1)} \in \mathbb{R}^{N_{l+1} \times F_{l+1}}$ and adjacency matrix $\A^{(l+1)} \in \mathbb{R}^{N_{l+1} \times N_{l+1}}$ (see Fig.~\ref{fig:lap_pyramid}).
To implement the MP layer, we consider a simple formulation that operates on the first-order neighbourhood of each node and accounts for the original node features through a separate set of weights acting as a layer-wise skip connection. 
The computation carried out by the $(j)$-th MP layer is given by
\begin{equation}
    \label{eq:mp}
    \begin{aligned}
    \X_{j+1} & = \textrm{MP}(\X_j, \A; \boldsymbol{\Theta}_{\textrm{MP}}) \\
             & = \text{ReLU}(\D^{-\frac{1}{2}} \A \D^{-\frac{1}{2}} \X_j \W + \X_j \mathbf{V}),
    \end{aligned}
\end{equation}
where $\boldsymbol{\Theta}_{\textrm{MP}} = \{\W \in \mathbb{R}^{F_j \times F_{j+1}}, \mathbf{V} \in \mathbb{R}^{F_j \times F_{j+1}} \}$ are the trainable weights relative to the mixing and skip component of the layer, respectively. 
Several other types of MP (\textit{e.g.}, those proposed in \cite{defferrard2016convolutional, kipf2016semi, velickovic2017graph, bianchi2019graph, xu2018powerful, hamilton2017inductive}) can seamlessly be used in conjunction with NDP pooling.
In the presence of annotated edges, the MP operation can be extended by following \cite{simonovsky2017dynamic} and \cite{schlichtkrull2018modeling}.

\section{Graph coarsening with Node Decimation Pooling}
\label{sec:method}

In this section, we describe the proposed NDP operation that consists of the three steps depicted in Fig.~\ref{fig:schema}: 
(a) decimate the nodes by dropping one of the two sides of the \maxcut{} partition; 
(b) connect the remaining nodes with a \textit{link construction} procedure; 
(c) sparsify the adjacency matrix resulting from the coarsened Laplacian, so that only \textit{strong} connections are kept, \textit{i.e.}, those edges whose weight is associated to an entry of the adjacency matrix above a given threshold $\epsilon$.
\begin{figure*}[!ht]
    \centering
    \includegraphics[keepaspectratio,width=\textwidth]{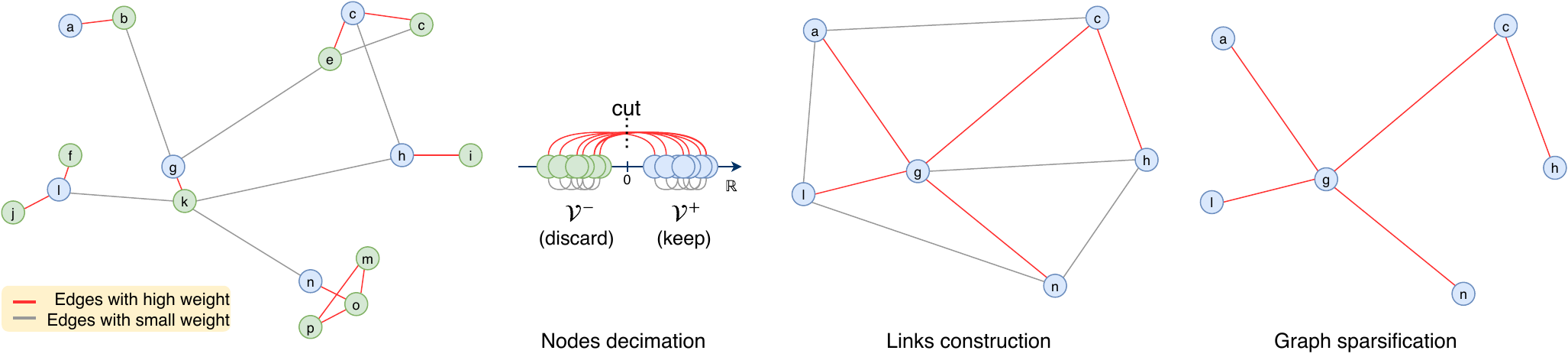}    
    \caption{Depiction of the proposed graph coarsening procedure. 
    First, the nodes are partitioned in two sets according to a \maxcut{} objective and then are decimated by dropping one of the two sets ($\mathcal{V}^{-}$). Then, a coarsened Laplacian is built by connecting the remaining nodes with a graph reduction procedure. 
    Finally, the edges with low weights in the new adjacency matrix obtained from the coarsened Laplacian are dropped to make the resulting graph sparser.}
    \label{fig:schema}
\end{figure*}
The proposed method is completely unsupervised and the coarsened graphs are pre-computed before training the GNN.

\subsection{Node decimation with MAXCUT spectral partitioning}
\label{sec:maxcut}

Similarly to pooling operations in Convolutional Neural Networks (CNNs) that compute local summaries of neighboring pixels, we propose a pooling procedure that provides an effective coverage of the whole graph and reduces the number of nodes approximately by a factor of 2.
This can be achieved by partitioning nodes in two sets, so that nodes in one set are strongly connected to the complement nodes of the partition, and then dropping one of the two sets.
The rationale is that strongly connected nodes exchange a lot of information after a MP operation and, as a result, they are highly dependent and their features become similar.
Therefore, one set alone can represent the whole graph sufficiently well.
This is similar to pooling in CNNs, where the maximum or the average is extracted from a small patch of neighboring pixels, which are assumed to be highly correlated and contain similar information.
In the following, we formalize the problem of finding the optimal subset of vertices that can be used to represent the whole graph.

The partition of the vertices (a cut) that maximizes the volume of edges whose endpoints are on opposite sides of the partition is the solution of the \maxcut{} problem~\cite{palagi2012computational}.
The \maxcut{} objective is expressed by the integer quadratic problem
\begin{equation}
    \label{eq:maxcut}
    \max \limits_{\z} \sum \limits_{i,j \in  \mathcal{V}} a_{ij}(1- z_i z_j) \;\; \text{s.t.} \;\; z_i \in \{ -1, 1 \},
\end{equation}
where $\z$ is the vector containing the optimization variables $z_i$ for $i=1, \dots, N$ indicating to which side of the bi-partition the node $i$ is assigned to; $a_{ij}$ is the entry at row $i$ and column $j$ of $\A$.
Problem \eqref{eq:maxcut} is NP-hard and heuristics must be considered to solve it.
The heuristic that gives the best-known \maxcut{} approximation in polynomial time is the Goemans-Williamson algorithm, which is based on the Semi-Definite Programming (SDP) relaxation~\cite{goemans1995improved}.
Solving SDP is cumbersome and requires specific optimization programs that scale poorly on large graphs.
Therefore, we hereby propose a simple algorithm based on the Laplacian spectrum.

First, we rewrite the objective function in \eqref{eq:maxcut} as a quadratic form of the graph Laplacian:
\begin{equation*}
    \begin{aligned}
    & \sum \limits_{i,j} a_{ij}(1- z_i z_j) = \sum \limits_{i,j} a_{ij} \left(\frac{z_i^2 + z_j^2}{2} - z_i z_j \right) \\
    & = \frac{1}{2} \sum \limits_i \Bigg[ \sum \limits_j a_{ij} \Bigg]z_i^2 + \frac{1}{2} \sum \limits_j \Bigg[ \sum \limits_i a_{ij} \Bigg] z_j^2 - \sum \limits_{i,j} a_{ij} z_i z_j \\
    & = \frac{1}{2} \sum \limits_i d_{ii} z_i^2 + \frac{1}{2} \sum \limits_j d_{jj} z_j^2 - \z^T \A \z \\
    & = \z^T \D \z - \z^T \A \z = \z^T \L \z.
    \end{aligned}
\end{equation*}

Then, we consider a continuous relaxation of the integer problem \eqref{eq:maxcut} by letting the discrete partition vector $\z$ assume continuous values, contained in a vector $\c$:
\begin{equation}
    \label{eq:relaxed}
    \max \limits_{\c} \; \c^T \L \c, \;\; \text{s.t.} \;\; \c \in \mathbb{R}^N \;\; \text{and} \;\; \| \c \|^2 = 1.
\end{equation}

Eq.~\ref{eq:relaxed} can be solved by considering the Lagrangian $\c^T \L \c + \lambda \c^T\c$ to find the maximum of $\c^T \L \c$ under constraint $\| \c \|^2 = 1$. 
By setting the gradient of the Lagrangian to zero, we recover the eigenvalue equation $\L \c + \lambda \c = 0$.
All the eigenvalues of $\L$ are non-negative and, by restricting the space of feasible solutions to vectors of unitary norm, the trivial solution $\c^* = \infty$ is excluded.
In particular, if $\| \c \|^2 = 1$, $\c^T \L \c$ is a Rayleigh quotient and reaches its maximum $\lambda_\text{max}$ (the largest eigenvalue of $\L$) when $\c^*$ corresponds to $\v_\text{max}$, the eigenvector associated with $\lambda_\text{max}$.

Since the components of $\v_\text{max}$ are real, we apply a rounding procedure to find a discrete solution.
Specifically, we are looking for a partition $\z^* \in \mathcal{Z}$, where $\mathcal{Z} = \{\z: \z \in \{-1,1\}^N\}$ is the set of all feasible cuts, so that $\z^*$ is the closest (in a least-square sense) to $\c^*$.
This amounts to solving the problem
\begin{equation}
    \z^* = \argmin \{ \| \c^* - \z \|_2: \z \in \mathcal{Z} \},
\end{equation}
with the optimum given by
\begin{equation}
\label{eq:partition_vec}
    z_i^* = 
    \begin{cases}
    1, & c_i^* \geq 0, \\
    -1, & c_i^* < 0.
    \end{cases}
\end{equation}


By means of the rounding procedure in \eqref{eq:partition_vec}, the nodes in $\mathcal{V}$ are partitioned in two sets, $\mathcal{V}^{+}$ and $\mathcal{V}^{-} = \mathcal{V} \setminus \mathcal{V}^{+}$, such that
\begin{equation}
\label{eq:partition}
    \mathcal{V}^{+} = \{i \in \mathcal{V}: \v_\text{max}[i] \geq 0 \}.
\end{equation}
In the NDP algorithm we always drop the nodes in $\mathcal{V}^{-}$, \textit{i.e.}, the nodes associated with a negative value in $\v_\text{max}$. 
However, it would be equivalent to drop the nodes in $\mathcal{V}^{+}$.
The node decimation procedure offers two important advantages: 
i) it removes approximately half of the nodes when applied, \textit{i.e.}, $|\mathcal{V}^{+}| \approx |\mathcal{V}^{-}|$; 
ii) the eigenvector $\v_\text{max}$ can be quickly computed with the power method~\cite{bianchi2017agent}.

There exists an analogy between the proposed spectral algorithm for partitioning the graph nodes and spectral clustering~\cite{von2007tutorial}.
However, spectral clustering solves a \texttt{\small{minCUT}} problem~\cite{shi2000normalized, icml2020_1614}, which is somehow orthogonal to the \maxcut{} problem considered here.
In particular, spectral clustering identifies $K \geq 2$ clusters of densely connected nodes by cutting the smallest volume of edges in the graph, while our algorithm cuts the largest volume of edges yielding two sets of nodes, $\mathcal{V}^{+}$ and $\mathcal{V}^{-}$, that cover the original graph in a similar way.
Moreover, spectral clustering partitions the nodes in $K$ clusters based on the values of the eigenvectors associated with the $M \geq 1$ smallest eigenvalues, while our algorithm partitions the nodes in two sets based only on the last eigenvector $\v_\text{max}$.

\subsection{Links construction on the coarsened graph}
\label{sec:links_construction}
After dropping nodes in $\mathcal{V}^{-}$ and all their incident edges, the resulting graph is likely to be disconnected.
Therefore, we use a link construction procedure to obtain a connected graph supported by the nodes in $\mathcal{V}^{+}$.
Specifically, we adopt the Kron reduction~\cite{kron_red} to generate a new Laplacian $\L^{(1)}$, which is computed as the Schur complement of $\L$ with respect to the nodes in $\mathcal{V}^{-}$.
In detail, the reduced Laplacian $\L^{(1)}$ is
\begin{equation}
\label{eq:kron}
\L^{(1)} = \L \setminus \L_{\mathcal{V}^{-}, \mathcal{V}^{-}} = \L_{\mathcal{V}^{+}, \mathcal{V}^{+}} - \L_{\mathcal{V}^{+}, \mathcal{V}^{-}} \L_{\mathcal{V}^{-}, \mathcal{V}^{-}}^{-1} \L_{\mathcal{V}^{-}, \mathcal{V}^{+}}
\end{equation}
where $\L_{\mathcal{V}^{+}, \mathcal{V}^{-}}$ identifies a sub-matrix of $\L$ with rows (columns) corresponding to the nodes in $\mathcal{V}^{+}$ ($\mathcal{V}^{-}$).

It is possible to show that $\L_{\mathcal{V}^{-}, \mathcal{V}^{-}}$ is always invertible if the associated adjacency matrix $\A$ is \textit{irreducible}.
We note that $\A$ is irreducible when the graph is not disconnected (\textit{i.e.}, has a single component), a property that holds when the algebraic multiplicity of the eigenvalue $\lambda_\text{min}=0$ is 1.

Let us consider the case where $\A$ has no self loops.
The Laplacian is by definition a weakly diagonally dominant matrix, since $\L_{ii} = \sum_{j=1, j \neq i}^N |\L_{ij}|$ for all $i \in \mathcal{V}$. 
If $\A$ is irreducible, then $\L$ is also irreducible. 
This implies that the strict inequality $\L_{ii} > \sum_{j=1, j \neq i}^n |\L_{ij}|$ holds for at least one vertex $i \in \mathcal{V}^{-}$. 
It follows that the Kron-reduced Laplacian $\L_{\mathcal{V}^{-},\mathcal{V}^{-}}$ is also irreducible, diagonally dominant, and has at least one row with a strictly positive row sum. Hence, $\L_{\mathcal{V}^{-},\mathcal{V}^{-}}$ is invertible, as proven by~\cite{horn2012matrix} in Corollary 6.2.27.
When $\A$ contains self-loops, the existence of the inverse of $\L_{\mathcal{V}^{-}, \mathcal{V}^{-}}$ is still guaranteed through a small work-around, which is discussed in App.~\ref{sec:kron_loops}.
Finally, if the graph is disconnected then $\A$ is reducible (\textit{i.e.}, it can be expressed in an upper-triangular block form by simultaneous row/column permutations); in this case, the Kron reduction can be computed by means of the generalized inverse $\L_{\mathcal{V}^{-}, \mathcal{V}^{-}}^\dagger$ and the solution corresponds to a generalized Schur complement of $\L$.

$\L^{(1)}$ in \eqref{eq:kron} is a well-defined Laplacian where two nodes are connected if and only if there is a path between them in $\L$ (connectivity preservation property). 
Also, $\L^{(1)}$ does not introduce self-loops and guarantees the preservation of resistance distance~\cite{shuman2016multiscale}.
Finally, Kron reduction guarantees spectral interlacing between the original Laplacian $\mathbf{L} \in \mathbb{R}^{N \times N}$ and the new coarsened one $\L^{(1)} \in \mathbb{R}^{N_1 \times N_1}$, with $N_1 \leq N$.
Specifically, we have $\lambda_i \geq \lambda^{(1)}_i \geq \lambda_{N-N_1+i}$, $\forall i=1, \dots, N_1$, where $\lambda_i$ and $\lambda^{(1)}_i$ are the eigenvalues of $\L$ and $\L^{(1)}$, respectively.

The adjacency matrix of the new coarsened graph is recovered from the coarsened Laplacian: 
\begin{equation}
    \A^{(1)} = \left( -\L^{(1)} + \text{diag}(\{ \sum_{j=1, j\neq i}^{N_1} \L^{(1)}_{ij}\}_{i=1}^{N_1}) \right).    
\end{equation}

We note that $\A^{(1)}$ does not contain self-loops;
we refer to App.~\ref{sec:kron_loops} for a discussion on how to handle the case with self-loops in the original adjacency matrix.

A pyramid of coarsened Laplacians is generated by recursively applying node decimation followed by Kron reduction. 
At the end of the procedure, the adjacency matrices $\mathcal{A} = \{\A^{(0)}, \A^{(1)}, \dots, \A^{(l)}, \dots\}$ of the coarsened graphs are derived from the associated coarsened Laplacians.
Note that we interchangeably refer with $\A^{(0)}$ or $\A$ to the original adjacency matrix.
The adjacency matrices in $\mathcal{A}$ are used to implement hierarchical pooling in deep GNN architectures.

\subsection{Graph sparsification}
\label{sec:sparsify}

To implement multiple pooling operations the graph must be coarsened several times.
Due to the connectivity preservation property, by repeatedly applying Kron reduction the graph eventually becomes fully-connected.
This implies a high computational burden in deeper layers of the network, since the complexity of MP operations scales with the number of edges.

To address this issue, it is possible to apply the spectral sparsification algorithm proposed in~\cite{batson2013spectral} to obtain a sparser graph. However, we found that this procedure leads to numerical instability and poor convergence during the learning phase.
Therefore, to limit the number of edges with non-zero weights we propose a sparsification procedure that removes from the adjacency matrix of the coarsened graph the edges whose weight is below a small user-defined threshold $\epsilon$:

\begin{equation}
    \bar \A^{(i)} = 
    \begin{cases}
    \bar a^{(l)}_{ij} = 0,              & \;\; \text{if} \;\; |a^{(l)}_{ij}| \leq \epsilon \\
    \bar a^{(l)}_{ij} = a^{(l)}_{ij},   &\;\; \text{otherwise}.\\
    \end{cases}
\end{equation}

\subsection{Pooling with decimation matrices.}
\label{sec:pooling_matrices}

To pool the node features with a differentiable operation that can be applied while training the GNN, we multiply the graph signal $\X^{(l)}$ with a \textit{decimation matrix} $\mathbf{S}^{(l)} \in \mathbb{N}^{N_{l+1} \times N_l}$.
$\mathbf{S}^{(l)}$ is obtained by selecting from the identity matrix $\I_{N_l} \in \mathbb{N}^{N_l \times N_l}$ the rows corresponding to the vertices in $\mathcal{V}^{+}$:
\begin{equation}
    \label{eq:dec_matr}
    \X^{(l+1)} = \mathbf{S}^{(l)} \X^{(l)} = \left[ \I_{N_l} \right]_{\mathcal{V}^{+},:}\X^{(l)}.
\end{equation}

As discussed in Sec.~\ref{sec:maxcut}, NDP approximately halves the the nodes of the current graph at each pooling stage.
This is a consequence of the \maxcut{} objective that splits the nodes in two sets so that the volume of edges crossing the partition, \textit{i.e.}, the edges to be cut, is maximized.
Intuitively, if one of the two sets is much smaller than the other, more edges are cut by moving some nodes to the smaller set.
For this reason, the application of a single decimation matrix $\mathbf{S}^{(l)}$ shares similarities with a classic pooling with stride 2 in CNNs.

It follows that a down-sampling ratio of $\approx2^k$ can be obtained in NDP by applying $k$ decimation matrices in cascade.
This enables moving from level $l$ to level $l+k$ ($k > 1$) in the pyramid of coarsened graphs.
Fig.~\ref{fig:lap_pyramid} shows an example of pooling with downsampling ratio $\approx8$, where the GNN performs message-passing on $\A^{(3)}$ right after $\A^{(0)}$.
\begin{figure}[!ht]
    \centering
    \includegraphics[keepaspectratio,width=.75\columnwidth]{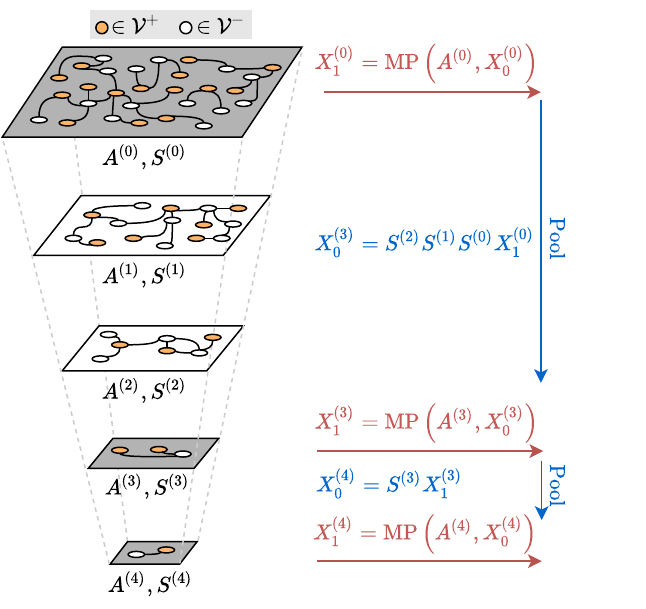}    
    \caption{This example shows how it is possible to skip some MP operations on intermediate levels of the pyramid of coarsened graphs. 
    Such a procedure shares analogies with pooling with a larger stride in traditional CNNs and can be considered as a higher-order graph pooling.
    After the first MP operation on $\A^{(0)}$, the node features are pooled by applying in cascade 3 decimation matrices, $\S^{(0)}$, $\S^{(1)}$, and $\S^{(2)}$.
    Afterwards, it is possible to directly perform a MP operation on $\A^{(3)}$, skipping the MP operations on $\A^{(1)}$ and $\A^{(2)}$.}
    \label{fig:lap_pyramid}
\end{figure}

\section{Analysis of the graph coarsening procedure and implementation details}
\label{sec:analysis}

\subsection{Numerical precision in eigendecomposition}
\begin{figure}[!ht]
    \centering
    \includegraphics[keepaspectratio,width=\columnwidth]{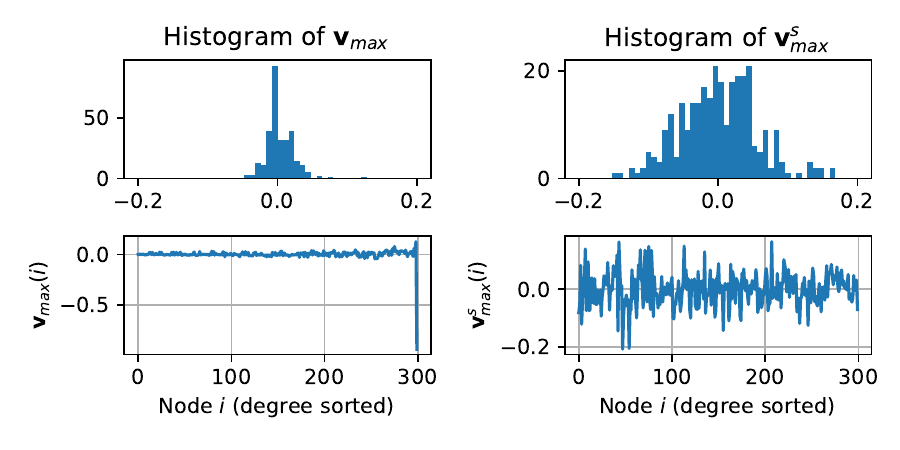}    
    \caption{(Left) distribution and values assumed by $\v_\text{max}$. (Right) distribution and values assumed by $\v_\text{max}$. The entries of the eigenvectors are sorted by node degree. A Stochastic Block Model graph was used in this example.}
    \label{fig:v_max}
\end{figure}
The entries of $\v_\text{max}$ associated with low-degree nodes assume very small values and their signs may be flipped due to numerical errors.
The partition obtained by using $\v^s_\text{max}$, i.e., the eigenvector of the symmetric Laplacian $\L_s$, in \eqref{eq:partition} is analytically the same.
Indeed, since $\v^s_\text{max} = \D^{-1/2} \v_\text{max}$, the values of the two eigenvectors are rescaled by positive numbers and, therefore, the sign of their components is the same.
However, a positive effect of the degree normalization is that the values in $\v^s_\text{max}$ associated to nodes with a low degree are amplified.

Fig.~\ref{fig:v_max} compares the values in the eigenvectors $\v_\text{max}$ and $\v^s_\text{max}$, computed on the same graph.
Since many values in $\v_\text{max}$ are concentrated around zero, partitioning the nodes according to the sign of the entries in $\v_\text{max}$ gives numerically unstable results.
On the other hand, since the values in $\v^s_\text{max}$ are distributed more evenly the nodes can be partitioned more precisely. 

Note that, even if the indexes of $\mathcal{V}^{+}$ are identified from the eigenvector of $\L_s$, the Kron reduction is still performed on the Laplacian $\L$.
In the supplementary material we report numerical differences in the size of the cut obtained on random graphs when using $\v_\text{max}$ or $\v^s_\text{max}$.

\subsection{Evaluation of the approximate \maxcut{} solution}
\label{sec:maxcut_eval}
Since computing the optimal \maxcut{} solution is NP-hard, it is generally not possible to evaluate the quality of the cut found by the proposed spectral method (Sect.~\ref{sec:maxcut}) in terms of discrepancy from the \maxcut{}.
Therefore, to assess the quality of a solution we consider the following bounds
\begin{equation}
    \label{eq:inequality}
    0.5 \leq \frac{\maxcut{}}{|\mathcal{E}|} \leq \frac{\lambda^\text{s}_\text{max}}{2} \leq 1.
\end{equation}
The value $\lambda^\text{s}_\text{max}/2$ is an upper-bound of $\maxcut{}/|\mathcal{E}|$, where $\lambda^\text{s}_\text{max}$ is the largest eigenvalue of $\L_s$ and $|\mathcal{E}| = \sum_{i,j} a_{ij}$.
The lower-bound $0.5$ is given by the \textit{random cut}, which uniformly assigns the nodes to the two sides of the partition.
\footnote{
\label{note_rc}
    Assume a graph without self-loops, i.e., $(i,i) \notin \mathcal{E}$.
    A random cut $\z$ is, on average, at least 0.5 of the optimal cut $\z^*$: $\mathbb{E}[|\z|] = \sum_{(i,j) \in \mathcal{E}} \mathbb{E}[z_i z_j] = \sum_{(i,j) \in \mathcal{E}} \text{Pr}[(i,j) \in \z] = \frac{|\mathcal{E}|}{2} \geq \frac{|\z^*|}{2}$.
} 
The derivation of the upper-bound is in App.~\ref{sec:upperbound_derivation}.

Let $\z$ be a partition vector such as the one in \eqref{eq:partition_vec}.
The cut size $\text{cut}(\z) = (\z^T \L \z)/2$ is the number of edges crossing the partition induced by $\z$ (derivation in App.~\ref{sec:upperbound_derivation}).
To measure the \emph{proportion} of edges cut by $\z$, we introduce the function $\gamma(\z)$, whose definition and bounds are the following:
\begin{equation}
    \label{eq:cut_size}
    0 \leq \gamma(\z) = \frac{\z^T \L \z}{2\sum_{i,j} a_{ij}} = \frac{2\cdot\text{cut}(\z)}{2|\mathcal{E}|} \leq \frac{\maxcut{}}{|\mathcal{E}|}.
\end{equation}
Note that $\gamma(\cdot)$ depends also on $\L$, but we keep it implicit to simplify the notation.

Let us now consider the best- and worst-case scenarios.
The best case is the bipartite graph, where the \maxcut{} is known and it cuts all the graph edges.
The partition $\z$ found by our spectral algorithm on bipartite graphs is optimal, \textit{i.e.}, $\gamma(\z) = \maxcut{}/|\mathcal{E}| = 1$.
In graphs that are close to be bipartite or, in general, that have a very sparse and regular connectivity, a large percentage of edges can be cut if the nodes are partitioned correctly.
Indeed, for these graphs the \maxcut{} is usually large and is closer to the upper-bound in \eqref{eq:inequality}.
On the other hand, in very dense graphs the \maxcut{} is smaller, as well as the gap between the upper- and lower-bound in \eqref{eq:inequality}. 
Notably, the worst-case scenario is a complete graph where is not possible to cut more than half of the edges, \textit{i.e.}, $\maxcut{}=0.5$.
We note that in graphs made of a sparse regular part that is weakly connected to a large dense part, the gaps in \eqref{eq:inequality} can be arbitrarily large.

The proposed spectral algorithm is not designed to handle very dense graphs; an intuitive explanation is that $\v^s_\text{max}$ can be interpreted as the graph signal with the highest frequency, since its sign oscillates as much as possible when transiting from a node to one of its neighbors. 
While such oscillation in the sign is clearly possible on bipartite graphs, in complete graphs it is not possible to find a signal that assumes an opposite sign on neighboring nodes, because all nodes are connected with each other.
Remarkably, the solution \eqref{eq:partition_vec} found by the spectral algorithm on very dense graphs can be worse than the random cut.
A theoretical result found by Trevisan~\cite{trevisan2012max} states that a spectral algorithm, like the one we propose, is guaranteed to yield a cut larger than the random partition only when $\lambda^s_\text{max} \geq 2(1-\tau) = 1.891$ (see App.~\ref{sec:trevisan} for details).

To illustrate how the size of the cut found by the spectral algorithm changes between the best- and worst-case scenarios, we randomly add edges to a bipartite graph until it becomes complete.
Fig.~\ref{fig:adding_edges} illustrates how the size of the cut $\gamma(\z)$ induced by the spectral partition $\z$ changes as more edges are added and the original structure of the graph is corrupted (blue line). The figure also reports the size of the random cut (orange line) and the \maxcut{} upper bound from Eq.~\eqref{eq:cut_size} (green line). The black line indicates the threshold from \cite{trevisan2012max}, i.e., the value of $\lambda^2_\text{max}/2$ below which the spectral cut is no longer guaranteed to be larger than the random cut.
The graph used to generate the figure is a regular grid; however, similar results hold also for other families of random graphs and are reported in the supplementary material.
\begin{figure}[!ht]
    \centering
    \includegraphics[keepaspectratio,width=.9\columnwidth]{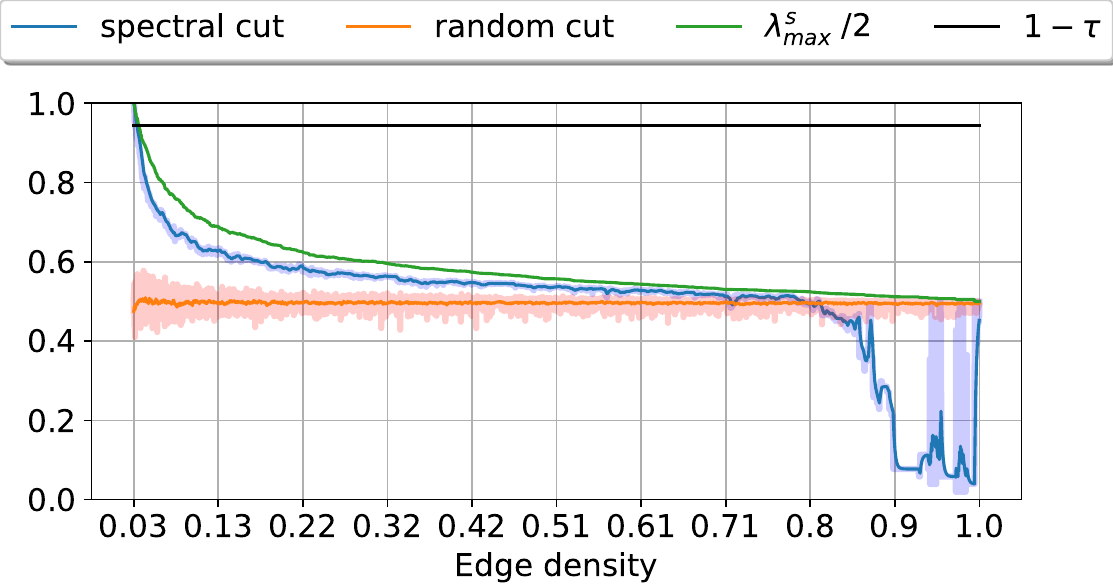}    
    \caption{
    Blue line: fraction of edges cut by the partition yield by the spectral algorithm. 
    Orange line: fraction of edges removed by a random cut. 
    Green line: the \maxcut{} upper bound as a function of the largest eigenvalue $\lambda^s_\text{max}$ of the symmetric Laplacian. 
    Black line: the threshold from \cite{trevisan2012max} indicating the value of $\lambda^s_\text{max}/2$ below which one should switch to the random cut to obtain a solution guaranteed to be $\geq 0.53\cdot$\maxcut{}. 
    The x-axis indicates the density of the graph connectivity, which increases by randomly adding edges.}
    \label{fig:adding_edges}
\end{figure}
%

Fig.~\ref{fig:adding_edges} shows that the spectral algorithm finds a better-than-random cut even when $\lambda^s_\text{max}/2 < 1-\tau$ (i.e., when the result from \cite{trevisan2012max} does not hold), and only approaches the size of the random cut when the edge density is very high (70\%-80\%). 

Importantly, when the size of the spectral partition becomes smaller than the random partition, the upper-bound $\lambda^s_\text{max}/2 \approx 0.5$, meaning that the random cut is very close to the \maxcut{}.
To obtain a cut that is always at least as good as the random cut, we first compute the partition $\z$ as in \eqref{eq:partition_vec} and evaluate its size $\gamma(\z)$: if $\gamma(\z) < 0.5$, we return a random partition instead.

We conclude by noticing that, due to the smoothing effect of MP operations, the nodes belonging to densely connected graph components are likely to have very similar representations computed by the GNN; it is, therefore, not important which of these nodes are dropped by a random cut.
The random cut in these cases not only is optimal in terms of the \maxcut{} objective, but it also introduces stochasticity that provides robustness when training the GNN model.

\subsection{Pseudocode}
\label{sec:pseudocode}

The procedure for generating the pyramid of coarsened adjacency matrices and the pooling matrices used for decimation is reported in Alg.~\ref{alg:coarsening}.
$\mathcal{L}$ is a list of positive integers indicating the levels in the pyramid of coarsened Laplacians that we want to compute. For instance, given levels $\mathcal{L} = [1, 3, 5]$ for a graph of $N$ nodes, the algorithm will return the coarsened graphs with approximately $N/2$, $N/8$, and $N/32$ nodes (in general, $N/2^{l_i}$ for each $l_i$ in $\mathcal{L}$). 
Matrix $\mathbf{R}$ is a buffer that accumulates decimation matrices when one or more coarsening levels are skipped. 
This happens when the GNN implements high order pooling, as discussed in Sect~\ref{sec:pooling_matrices}.

\begin{algorithm}\footnotesize
\caption{Graph coarsening}
\label{alg:coarsening}
\begin{algorithmic}[1]
\REQUIRE adjacency matrix $\A$, coarsening levels $\mathcal{L}$, sparsification threshold $\epsilon$
\ENSURE coarsened adjacency matrices $\mathcal{\bar A}$, decimation matrices $\mathcal{S}$
\STATE $\A^{(0)} = \A$, $\mathbf{R} = \I_N$, $\mathcal{A} = \{\}$, $\mathcal{S} = \{\}$, $l=0$
\WHILE{$l \leq \text{max}(\mathcal{L})$}
\STATE $\A^{(l+1)}, \S^{(l+1)} = \text{pool}(\A^{(l)})$ \label{line:pool}
\IF{$l \in \mathcal{L}$}
\STATE $\mathcal{A} = \mathcal{A} \cup \A^{(l+1)}$, $\mathcal{S} = \mathcal{S} \cup \S^{(l+1)}\mathbf{R}$
\STATE $\mathbf{R} = \I_{N_l}$
\ELSE
\STATE $\mathbf{R} = \S^{(l+1)}\mathbf{R}$
\ENDIF
\STATE $l = l+1$
\ENDWHILE
\STATE $\mathcal{\bar A} = \{\bar \A^{(l)}: \bar a^{(l)}_{ij} = a^{(l)}_{ij}$ if $a^{(l)}_{ij} > \epsilon$ and $0$ otherwise, $\forall \A^{(l)} \in \mathcal{A} \}$ 
\end{algorithmic}
\end{algorithm}

Alg.~\ref{alg:pooling} shows the details of the pooling function, used in line~\ref{line:pool} of Alg.~\ref{alg:coarsening}.

\begin{algorithm}\footnotesize
\caption{$\text{pool}(\cdot)$ function}
\label{alg:pooling}
\begin{algorithmic}[1]
\REQUIRE adjacency matrix $\A^{(l)} \in \mathbb{R}^{N_l \times N_l}$
\ENSURE coarsened adjacency matrix $\A^{(l+1)} \in \mathbb{R}^{N_{l+1} \times N_{l+1}}$, decimation matrix $\S^{(l+1)} \in \mathbb{N}^{N_{l+1} \times N_l}$
\STATE get $\L^{(l)} = \D^{(l)} - \A^{(l)}$  and $\L_s^{(l)} = \I - (\D^{(l)})^{-\frac{1}{2}} \A^{(l)} (\D^{(l)})^{-\frac{1}{2}}$ 
\STATE compute the eigenvector $\v^s_\text{max}$ of $\L_s^{(l)}$
\STATE partition vector $\z$ s.t. $z_i = 1$ if $\v^s_\text{max}[i] \geq 0$,  $z_i = -1$ if $\v^s_\text{max}[i] < 0$
\IF{$\gamma(\z) < 0.5$}
\STATE random sample $z_i \sim \{-1,1\}, \forall i=1, \dots, N_l$ (random cut)
\ENDIF
\STATE $\mathcal{V}^{+} = \{i: z_i=1\}$, $\mathcal{V}^{-} = \{i: z_i=-1\}$
\STATE $\L^{(l+1)} = \L^{(l)} \setminus \L^{(l)}_{\mathcal{V}^{-}, \mathcal{V}^{-}}$ (Kron reduction)
\STATE $\A^{(l+1)} = -\L^{(l+1)} + \text{diag}(\sum_{j\neq i} \L^{(l+1)}_{ij})$
\STATE $\S^{(l+1)} = [\I_{N_{l+1}}]_{\mathcal{V}^{+}, :}$
\end{algorithmic}
\end{algorithm}

\subsection{Computational cost analysis}
The most expensive operations in the NDP algorithm are i) the cost of computing the eigenvector $\v^s_\text{max}$, and ii) the cost of inverting the submatrix $\L_{\mathcal{V}^-, \mathcal{V}^-}$ within the Kron reduction.

Computing all eigenvectors has a cost $\mathcal{O}(N^3)$, where $N$ is the number of nodes. However, computing only the eigenvector corresponding to the largest eigenvalue is fast when using the power method \cite{watkins2004fundamentals}, which requires only few iterations (usually 5-10), each one of complexity $\mathcal{O}(N^2)$.
The cost of inverting $\L_{\mathcal{V}^-, \mathcal{V}^-}$ is $\mathcal{O}(|\mathcal{V}^-|^3)$, where $|\mathcal{V}^-|$ is the number of nodes that are dropped. 

We notice that the coarsened graphs are pre-computed before training the GNN.
Therefore, the computational time of graph coarsening is much lower compared to training the GNN for several epochs, since each MP operation in the GNN has a cost $\mathcal{O}(N^2)$.

\subsection{Structure of the sparsified graphs}
When applying the sparsification, the spectrum of the resulting adjacency matrix $\bar{\A}$ is preserved, up to a small factor that depends on $\epsilon$, with respect to the spectrum of $\A$.

\begin{theorem}
Let $\mathbf{Q}$ be a matrix used to remove small values in the adjacency matrix $\A$, which is defined as
\begin{equation}
    \mathbf{Q} = 
    \begin{cases}
    q_{ij} = - a_{ij}, & \;\; \text{if} \;\; |a_{ij}| \leq \epsilon \\
    q_{ij} = 0, &\;\; \text{otherwise}.\\
    \end{cases}
\end{equation}
Each eigenvalue $\bar \alpha_i$ of the sparsified adjacency matrix $\mathbf{\bar \A} = \A + \mathbf{Q}$ is bounded by 
\begin{equation}
    \bar \alpha_i \leq \alpha_i + \u_i^T \mathbf{Q} \u_i, 
\end{equation}
where $\alpha_i$ and $\u_i$ are eigenvalue-eigenvector pairs of $\A$.
\end{theorem}

\begin{proof}
Let $\mathbf{P}$ be a matrix with elements $p_{ij} = \text{sign}(q_{ij})$ and consider the perturbation $\A + \epsilon \mathbf{P}$, which modifies the eigenvalue problem $\A \u_i = \alpha_i \u_i$ in 
\begin{equation}
\label{eq:eigen_prob}
    (\A + \epsilon \mathbf{P})(\u_i + \u_{\epsilon}) = (\alpha_i + \alpha_{\epsilon}) (\u_i + \u_{\epsilon}).
\end{equation}
where $\alpha_{\epsilon}$ is a small number and $\u_{\epsilon}$ a small vector, which are unknown and indicate a  perturbation on $\alpha_i$ and $\u_i$, respectively.
By expanding \eqref{eq:eigen_prob}, then canceling the equation $\A \u_i = \alpha_i \u_i$ and the high order terms $\mathcal{O}(\epsilon^2)$, one obtains

\begin{equation}
\label{eq:expansion1}
    \A\u_{\epsilon} + \epsilon \mathbf{P}\u_i = \alpha_i \u_{\epsilon} + \alpha_{\epsilon}\u_i.
\end{equation}

Since $\A$ is symmetric, its eigenvectors can be used as a basis to express the small vector $\u_{\epsilon}$
\begin{equation}
    \label{eq:basis}
    \u_{\epsilon} = \sum \limits_{j=1}^N \delta_j \u_j,
\end{equation}
where $\delta_j$ are (small) unknown coefficients. 
Substituting \eqref{eq:basis} in \eqref{eq:expansion1} and bringing $\A$ inside the summation, gives
\begin{equation}
    \label{eq:expansion2}
    \sum \limits_{j=1}^N \delta_j \A \u_j + \epsilon \mathbf{P}\u_i =   \alpha_i \sum \limits_{j=1}^N \delta_j \u_j + \alpha_{\epsilon}\u_i.
\end{equation}

By considering the original eigenvalue problem that gives $\sum \limits_{j=1}^N \delta_j \A \u_j = \sum \limits_{j=1}^N \delta_j \alpha_j \u_j$ and by left-multiplying each term with $\u_i^T$, \eqref{eq:expansion2} becomes
\begin{equation}
    \label{eq:expansion3}
    \u_i^T \sum \limits_{j=1}^N \delta_j \alpha_j \u_j + \u_i^T \epsilon \mathbf{P}\u_i =   \u_i^T \alpha_i \sum \limits_{j=1}^N \delta_j \u_j + \u_i^T \alpha_{\epsilon}\u_i.
\end{equation}
Since eigenvectors are orthogonal, $\u_i^T \u_j = 0, \forall j \neq i$ and $\u_i^T \u_j = 1$, for $j = i$, Eq. \eqref{eq:expansion3} becomes
\begin{equation}
    \begin{aligned}
    \u_i^T \delta_i \alpha_i \u_i + \u_i^T \epsilon \mathbf{P}\u_i & = \u_i^T \alpha_i \delta_i \u_i + \u_i^T \alpha_i\u_i, \\
    \u_i^T \epsilon \mathbf{P}\u_i & =  \u_i^T \alpha_{\epsilon}\u_i = \alpha_{\epsilon},
    \end{aligned}
\end{equation}
which, in turn, gives 
\begin{equation}
    \alpha_{\epsilon} = \u_i^T \epsilon \mathbf{P} \u_i \geq \u_i^T \mathbf{Q} \u_i,
\end{equation}
as $\mathbf{Q} \leq \epsilon \mathbf{P}$.
\end{proof}

A common way to measure the similarity of two graphs is to compare the spectrum of their Laplacians.
To extend the results of Theorem 1 to the spectra of the Laplacians $\L$ and $\bar \L$, respectively associated with the original and sparsified adjacency matrices $\A$ and $\bar \A$, it is necessary to consider the relationships between the eigenvalues of $\A$ and $\L$.
For a $d$-regular graph, the relationship $\lambda_i = d - \alpha_i$ links the $i$-th eigenvalue $\lambda_i$ of $\L$ to the $i$-th eigenvalue $\alpha_i$ of $\A$~\cite{lutzeyer2017comparing}.
However, for a general graph it is only possible to derive a loose bound, $d_\text{max} - \alpha_n \leq \lambda_n \leq d_\text{max} - \alpha_1$, that depends on the maximum degree $d_\text{max}$ of the graph~\cite[Lemma 2.21]{zumstein2005comparison}.

Therefore, we numerically compare the spectra of the Laplacians associated with the matrices in $\mathbf{A}$ and $\mathcal{\bar \A}$. 
In particular, Fig.~\ref{fig:spectra_diff}(\textit{top-left}) depicts the spectrum of the Laplacian associated to the original graph $\A^{(0)}$ (black dotted line) and the spectra $\Lambda(\L^{(1)})$, $\Lambda(\L^{(2)})$, $\Lambda(\L^{(3)})$ of the Laplacians associated with $\A^{(1)}$, $\A^{(2)}$, and $\A^{(3)}$.
Fig.~\ref{fig:spectra_diff}(\textit{top-right}) depicts the spectra of the Laplacians $\bar \L^{(1)}$, $\bar \L^{(2)}$, $\bar \L^{(3)}$ associated with the sparsified matrices $\bar \A^{(1)}$, $\bar \A^{(2)}$, and $\bar \A^{(3)}$.
It is possible to observe that the spectra of $\L^{(l)}$ and $\bar \L^{(l)}$ are almost identical and therefore, to better visualize the differences, we show in Fig.~\ref{fig:spectra_diff}(\textit{bottom}) the absolute differences $|\Lambda (\L^{(l)}) - \Lambda(\bar \L^{(l)})|$.
The graph used in Fig.~\ref{fig:spectra_diff} is a random sensor network and the sparsification threshold is $\epsilon = 10^{-2}$, which is the one adopted in all our experiments.
\begin{figure}[!ht]
    \centering
    \subfigure{
        \includegraphics[width=0.48\columnwidth]{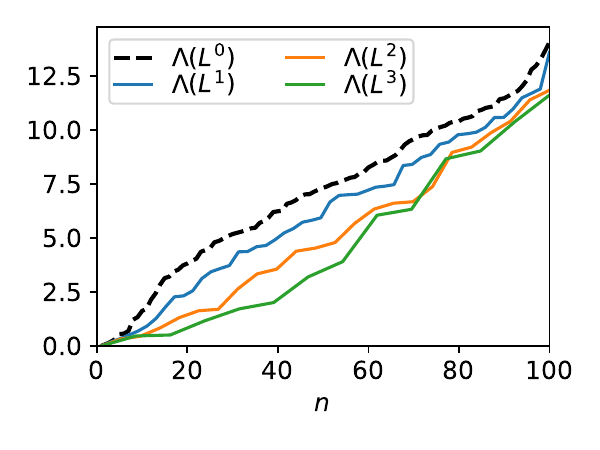}
    }\hspace{-.6cm}%
    ~
    \subfigure{
        \includegraphics[width=0.48\columnwidth]{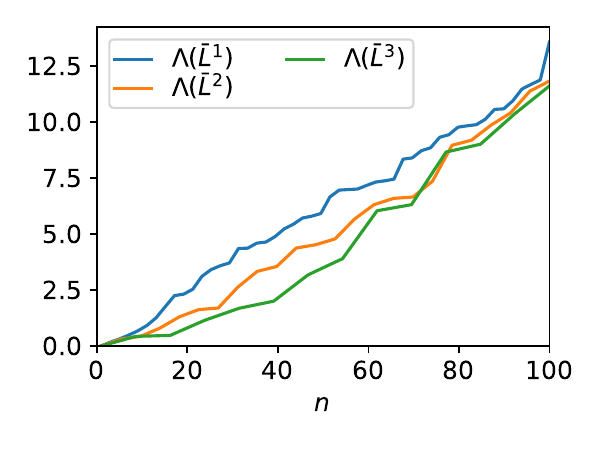}
    }
    \subfigure{
        \includegraphics[width=.7\columnwidth]{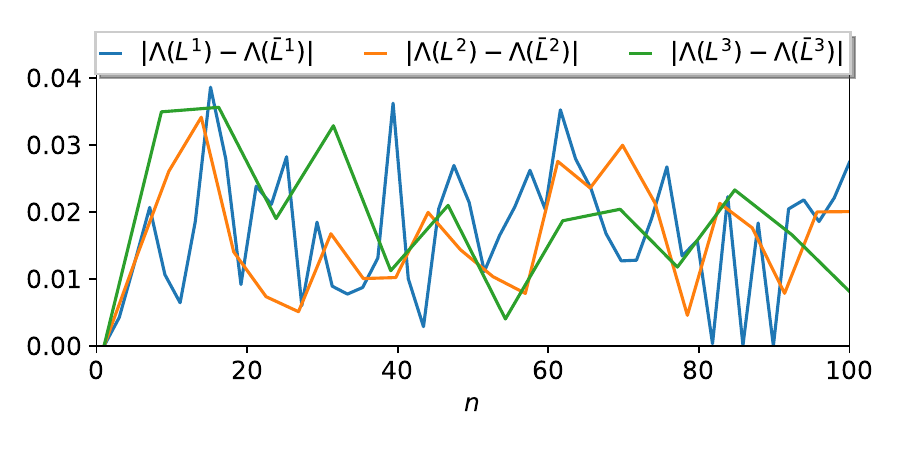}
    }

    \caption{
    \textit{Top-left}: Spectrum of the Laplacians associated with the original adjacency $\A^{(0)}$ and the coarsened versions $\A^{(1)}$, $\A^{(2)}$, and $\A^{(3)}$ obtained with the NDP algorithm. 
    \textit{Top-right}: Spectrum of the Laplacians associated with the sparsified adjacency matrices $\bar \A^{(1)}$, $\bar \A^{(2)}$, and $\bar \A^{(3)}$.
    \textit{Bottom}: Absolute difference between the spectra of the Laplacians.
    }
    \label{fig:spectra_diff}
\end{figure}

To quantify how much the coarsened graph changes as a function of $\epsilon$, we consider the \emph{spectral distance} that measures a dissimilarity between the spectra of the Laplacians associated with $\A$ and $\bar \A$~\cite{loukas2019graph}. 
The spectral distance is computed as 
\begin{equation}
    SD(\L, \bar \L; \epsilon) = \frac{1}{K} \sum_{k=2}^{K+1} \frac{|\bar\lambda_k(\epsilon) - \lambda_k|}{\lambda_k},
\end{equation}
where $\{ \lambda_k \}_{k=2}^{K+1}$ and $\{ \bar \lambda_k(\epsilon) \}_{k=2}^{K+1}$ are, respectively, the $K$ smallest non-zero eigenvalues of $\L$ and $\bar{\L}$.

\begin{figure}[!ht]
	    \centering
        \includegraphics[keepaspectratio,width=0.3\textwidth]{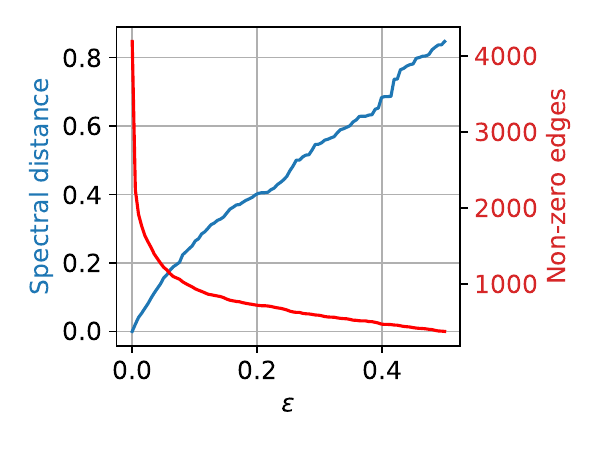}
	    \caption{In blue, the variation of spectral distance between the Laplacian $\L$ and the Laplacian $\bar{\L}$, associated with the adjacency matrix $\A$ sparsified with threshold $\epsilon$. In red, the number of edges that remain in $\bar{\L}$.}
	    \label{fig:varying_eps}
\end{figure}
Fig.~\ref{fig:varying_eps} depicts in blue the variation of spectral distance between $\L$ and $\bar{\L}$, as we increase the threshold $\epsilon$ used to compute $\bar \A$. 
The red line indicates the number of edges that remain in $\bar{\A}$ after sparsification.
It is possible to see that for small increments of $\epsilon$ the spectral distance increases linearly, while the number of edges in the graph drops exponentially.
Therefore, with a small $\epsilon$ it is possible to discard a large amount of edges with minimal changes in the graph spectrum.

The graph used to generate Fig.~\ref{fig:varying_eps} is a sensor network; results for other types of graph are in the supplementary material.

\section{Related work on graph pooling}
\label{sec:related_work}

We discuss related work on GNN pooling by distinguishing between \textit{topological} and \textit{feature-based} pooling methods.

\paragraph{Topological pooling methods}
similarly to NDP, topological pooling methods pre-compute coarsened graphs before training based on their topology.
Topological pooling methods are usually unsupervised, as they define how to coarsen the graph outside of the learning procedure. 
The GNN is then trained to fit its node representations to these pre-determined structures.
Pre-computing graph coarsening not only makes the training much faster by avoiding to perform graph reduction at every forward pass, but it also provides a strong inductive bias that prevents degenerate solutions, such as entire graphs collapsing into a single node or entire graph sections being discarded.
This is important when dealing with small datasets or, as we show in the following section, in tasks such as graph signal classification.

The approach that is most related to NDP and that has been adopted in several GNN architectures to perform pooling~\cite{bruna2013spectral, defferrard2016convolutional, monti2017geometric, fey2018splinecnn, levie2017cayleynets}, consists of coarsening the graph with GRACLUS, a hierarchical spectral clustering algorithm~\cite{dhillon2004kernel}. 
At each level $l$, two vertices $i$ and $j$ are clustered together in a new vertex $k$.
Then, a standard pooling operation (average or max pool) is applied to compute the node feature $\x^{(l+1)}_k$ from $\x^{(l)}_i$ and $\x^{(l)}_j$.
This approach has several drawbacks.
First, the connectivity of the original graph is not preserved in the coarsened graphs and the spectrum of their associated Laplacians is usually not contained in the spectrum of the original Laplacian.
Second, GRACLUS pooling adds ``fake'' nodes so that they can be exactly halved at each pooling step; this not only injects noisy information in the graph signal, but also increases the computational complexity in the GNN. 
Finally, clustering depends on the initial ordering of the nodes, which hampers stability and reproducibility.


An alternative approach is to directly cluster the rows (or the columns) of the adjacency matrix, as done by approach proposed in~\cite{bacciu2019non}, which decomposes $\A$ in two matrices $\mathbf{W} \in \mathbb{R}^{N \times K}$ and $\mathbf{H} \in \mathbb{R}^{K \times N}$ using the Non-negative Matrix Factorization (NMF) $\A \approx \mathbf{W}\mathbf{H}$.
The NMF inherently clusters the columns of $\A$ since the minimization of the NMF objective is equivalent to the objective in $k$-means clustering~\cite{ding2005equivalence}.
In particular, $\mathbf{W}$ is interpreted as the cluster representatives matrix and $\mathbf{H}$ as a soft-cluster assignment matrix of the columns in $\A$. 
Therefore, the pooled node features and the coarsened graph can be obtained as $\X^{(1)} = \mathbf{H}^T\X$ and $\A^{(1)} = \mathbf{H}^T \A \mathbf{H}$, respectively.
The main drawback is that NMF does not scale well to large graphs.


\paragraph{Feature-based pooling methods}
these methods compute a coarsened version of the graph through differentiable functions, which are parametrized by weights that are optimized for the task at hand.
Differently from topological pooling, these methods account for the node features, which change as the GNN is trained.
While this gives more flexibility in adapting the coarsening on the data and the task at hand, GNNs with feature-based pooling have more parameters; as such, training is slower and more difficult.

\textit{DiffPool}~\cite{ying2018hierarchical} is a pooling method that learns differentiable soft assignments to cluster the nodes at each layer. 
DiffPool uses two MP layers in parallel: one to update the node features, and one to generate soft cluster assignments.
The original adjacency matrix acts as a prior when learning the cluster assignments, while an entropy-based regularization encourages sparsity in the cluster assignments.
The application of this method to large graphs is not practical, as the cluster assignment matrix is dense and its size is $N \times K$, where $K$ is the number of nodes of the coarsened graph.

A second approach, dubbed \textit{Top-$K$} pooling~\cite{graphunet, cangea2018towards}, learns a projection vector that is applied to each node feature to obtain a score. 
The nodes with the $K$ highest scores are retained, while the remaining ones are dropped.
Since the top-$K$ selection is not differentiable, the scores are also used as a gating for the node features, allowing gradients to flow through the projection vector during backpropagation.
Top-$K$ is more memory efficient than DiffPool as it avoids generating cluster assignments. 
A variant proposed in~\cite{knyazev2019understanding} introduces in Top-$K$ pooling a soft attention mechanism for selecting the nodes to retain.
Another variant of Top-$K$, called SAGPool, processes the node features with an additional MP layer before using them to compute the scores~\cite{pmlr-v97-lee19c}.


\section{Experiments}
\label{sec:experiments}

We consider two tasks on graph-structured data: graph classification and graph signal classification.
The code used in all experiments is based on the Spektral library \cite{grattarola2020graph}, and the code to replicate all experiments of this paper is publicly available at GitHub.\footnote{\url{github.com/danielegrattarola/decimation-pooling}}

\subsection{Graph classification}
In this task, the $i$-th sample is a graph represented by the pair $\{ \A_i, \X_i \}$ which must be classified with a label $\mathbf{y}_i$.
We consider 2 synthetic datasets (\textit{Bench-easy} and \textit{Bench-hard})\footnote{\url{https://github.com/FilippoMB/Benchmark_dataset_for_graph_classification}} and 8 datasets of real-world graphs: \textit{Proteins}, \textit{Enzymes}, \textit{NCI1}, \textit{MUTAG}, \textit{Mutagenicity}, \textit{D\&D}, \textit{COLLAB}, and \textit{Reddit-Binary}.\footnote{\url{http://graphlearning.io}}
When node features $\X$ are not available, we use node degrees and clustering coefficients as a surrogate.
Moreover, we also use node labels as node features whenever they are available.

In the following, we compare NDP with GRACLUS~\cite{defferrard2016convolutional}, NMF~\cite{bacciu2019non}, DiffPool~\cite{ying2018hierarchical}, and Top-$K$~\cite{graphunet}.
In each experiment we adopt a fixed network architecture, MP(32)-P(2)-MP(32)-P(2)-MP(32)-AvgPool-Softmax, where MP(32) stands for a MP layer as described in \eqref{eq:mp} configured with 32 hidden units and ReLU activations, P(2) is a pooling operation with stride 2, AvgPool is a global average pooling operation on all the remaining graph nodes, and Softmax indicates a dense layer with Softmax activation.
As training algorithm, we use Adam~\cite{kingma2014adam} with initial learning rate 5e-4 and L\textsubscript{2} regularization with weight 5e-4.
As an exception, for the Enzymes dataset we used MP(64). 

Additional baselines are the Weisfeiler-Lehman (WL) graph kernel~\cite{shervashidze2011weisfeiler}, a GNN with only MP layers (\textit{Flat}), and a network with only dense layers (\textit{Dense}).
The comparison with \textit{Flat} helps to understand whether pooling operations are useful for a given task.
The results obtained by \textit{Dense}, instead, help to quantify how much additional information is brought by the graph structure compared to considering the node features alone.
While recent graph kernels~\cite{yanardag2015deep, martino2019hyper, togninalli2019wasserstein} and GNN architectures~\cite{bai2020learning, jiang2019walk} could be considered as further baselines for graph classification, the focus of our analysis and discussion is on graph pooling operators and, therefore, we point the interested reader towards the referenced papers. 

To train the GNN on mini-batches of graphs with a variable number of nodes, we consider the disjoint union of the graphs in each mini-batch and train the GNN on the combined Laplacians and graph signals.
See the supplementary material for an illustration.

\bgroup
\def\arraystretch{1} 
\setlength\tabcolsep{1em} 
\begin{table*}[!ht]
\footnotesize
\centering
\caption{Graph classification accuracy. Significantly better results ($p < 0.05$) are in bold.} 
\label{tab:graph_class}
\begin{tabular}{lcccccccc}
\cmidrule[1.5pt]{1-9}
\textbf{Dataset} & \textbf{WL} & \textbf{Dense} & \textbf{Flat} & \textbf{Diffpool} & \textbf{Top-$K$}  & \textbf{GRACLUS} & \textbf{NMF} & \textbf{NDP}\\
\cmidrule[.5pt]{1-9}
Bench-easy    & 92.6 & 29.3\t{$\pm$0.3} & 98.5\t{$\pm$0.3} & \textbf{98.6\t{$\pm$0.4}} & 82.4\t{$\pm$8.9} & 97.5\t{$\pm$0.5} & 97.4\t{$\pm$0.8} & 97.4\t{$\pm$0.9}\\ 
Bench-hard    & 60.0 & 29.4\t{$\pm$0.3} & 67.6\t{$\pm$2.8} & 69.9\t{$\pm$1.9} & 42.7\t{$\pm$15.2}  & 69.0\t{$\pm$1.5} & 68.6\t{$\pm$1.6} & \textbf{71.9\t{$\pm$0.8}}\\  
Proteins      & 71.2\t{$\pm$2.6}  & 68.7\t{$\pm$3.3} & 72.6\t{$\pm$4.8} & 72.8\t{$\pm$3.5} & 69.6\t{$\pm$3.3} & 70.3\t{$\pm$2.6} & 71.6\t{$\pm$4.1} & \textbf{73.4\t{$\pm$3.1}}\\ 
Enzymes       & 33.6\t{$\pm$4.1}  & 45.7\t{$\pm$9.9} & \textbf{52.0\t{$\pm$12.3}} & 24.6\t{$\pm$5.3} & 31.4\t{$\pm$6.8} & 42.0\t{$\pm$6.7} & 39.9\t{$\pm$3.6} & 44.5\t{$\pm$7.4}\\ 
NCI1          & \textbf{81.1\t{$\pm$1.6}}  & 53.7\t{$\pm$3.0} & 74.4\t{$\pm$2.5} & 76.5\t{$\pm$2.2} & 71.8\t{$\pm$2.6} & 69.5\t{$\pm$1.7} & 68.2\t{$\pm$2.2} & 74.2\t{$\pm$1.7}\\ 
MUTAG         & 78.9\t{$\pm$13.1} & \textbf{91.1\t{$\pm$7.1}} & 87.1\t{$\pm$6.6} & 90.5\t{$\pm$3.9} & 85.5\t{$\pm$11.0} & 84.9\t{$\pm$8.1} & 76.7\t{$\pm$14.4} & 87.9\t{$\pm$5.7}\\ 
Mutagenicity  & \textbf{81.7\t{$\pm$1.1}}  & 68.4\t{$\pm$0.3} & 78.0\t{$\pm$1.3} & 77.6\t{$\pm$2.7} & 71.9\t{$\pm$3.7} & 74.4\t{$\pm$1.8} & 75.5\t{$\pm$1.7} & 77.9\t{$\pm$1.4}\\ 
D\&D          & 78.6\t{$\pm$2.7}  & 70.6\t{$\pm$5.2} & 76.8\t{$\pm$1.5} & \textbf{79.3\t{$\pm$2.4}} & 69.4\t{$\pm$7.8} & 70.5\t{$\pm$4.8} & 70.6\t{$\pm$4.1} & 72.8\t{$\pm$5.4}\\ 
COLLAB        & 74.8\t{$\pm$1.3}  & 79.3\t{$\pm$1.6} & \textbf{82.1\t{$\pm$1.8}} & 81.8\t{$\pm$1.4} & 79.3\t{$\pm$1.8} & 77.1\t{$\pm$2.1} & 78.5\t{$\pm$1.8} & 79.1\t{$\pm$1.3} \\ 
Reddit-Binary & 68.2\t{$\pm$1.7}  & 48.5\t{$\pm$2.6} & 80.3\t{$\pm$2.6} & 86.8\t{$\pm$2.1} & 74.7\t{$\pm$4.5} & 79.2\t{$\pm$0.4} & 52.0\t{$\pm$2.1} & \textbf{88.0\t{$\pm$1.4}} \\ 
\cmidrule[1.5pt]{1-9}
\end{tabular}
\end{table*}
\egroup

We evaluate the model's performance by splitting the dataset in 10 folds. 
Each fold is, in turn, selected as the test set, while the remaining 9 folds become the training set.
For each different train/test split, we set aside 10\% of the training data as validation set, which is used for early stopping, i.e., we interrupt the training procedure after the loss on the validation set does not decrease for 50 epochs.

We report in Table~\ref{tab:graph_class} the test accuracy averaged over the 10 folds.
We note that no architecture outperforms every other in all tasks.
The WL kernel achieves the best results on NCI1 and Mutagenicity, but it does not perform well on the other datasets.
Interestingly, the \textit{Dense} architecture achieves the best performance on MUTAG, indicating that in this case, the connectivity of the graps does not carry useful information for the classification task.
The performance of the \textit{Flat} baseline indicates that in Enzymes and COLLAB pooling operations are not necessary to improve the classification accuracy.

NDP consistently achieves a higher accuracy compared to GRACLUS and NMF, which are also topological pooling methods.
We argue that the lower performance of GRACLUS is due to the fake nodes, which introduce noise in the graphs.
Among the two feature-based pooling methods, DiffPool always outperforms Top-$K$.
The reason is arguably that Top-$K$ drops entire parts of the graphs, thus discarding important information for the classification~\cite{knyazev2019understanding, icml2020_1614}.

In Fig.~\ref{fig:train_time}, we report the training time for the five different pooling methods.
As expected, GNNs configured with GRACLUS, NMF, and NDP are much faster to train compared to those based on DiffPool and Top$K$, with NDP being slightly faster than the other two topological methods.
\begin{figure*}
    \centering
    \includegraphics[keepaspectratio, width=0.9\textwidth]{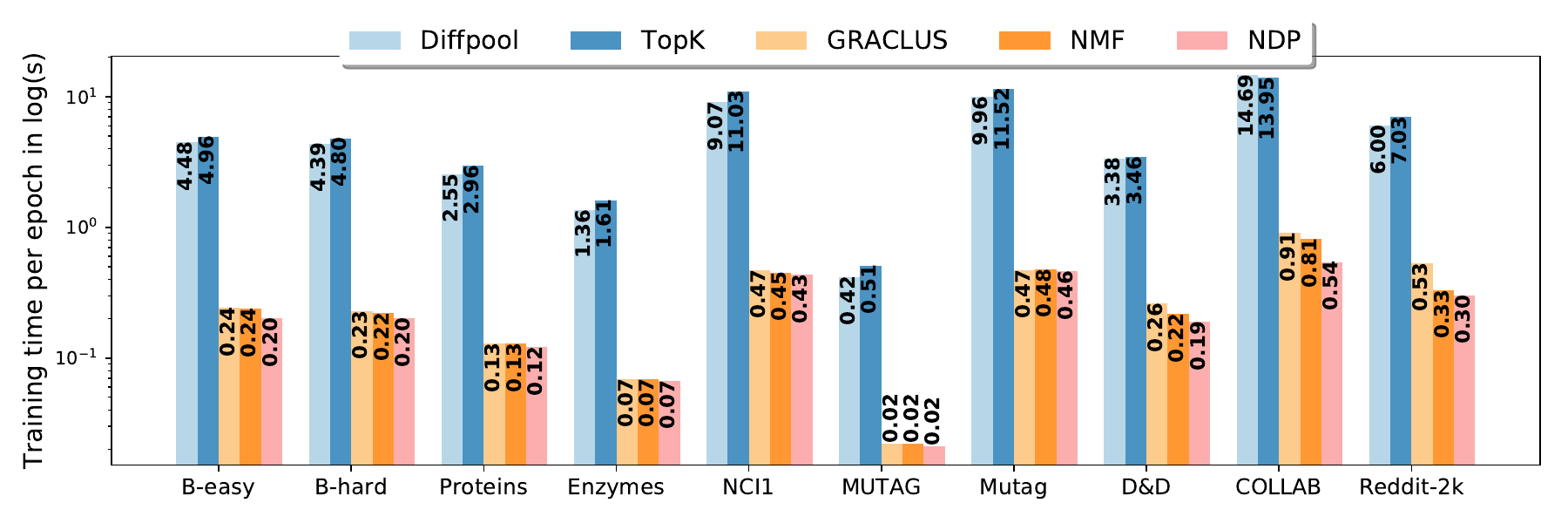}    
    \caption{Average training time per epoch (in seconds) for different pooling methods. The bars height is in logarithmic scale. Simulations were performed with an Nvidia RTX 2080 Ti.}
    \label{fig:train_time}
\end{figure*}
In Fig.~\ref{fig:acc_vs_time}, we plot the average training time per epoch against the average accuracy obtained by each pooling method on the 10 datasets taken into account. 
The scatter plot is obtained from the data reported in Tab.~\ref{tab:graph_class} and Fig.~\ref{fig:train_time}.
On average, NDP obtains the highest classification accuracy, slightly outperforming even Diffpool, while being, at the same time, the fastest among all pooling methods.
\begin{figure}
    \centering
    \includegraphics[keepaspectratio, width=0.5\columnwidth]{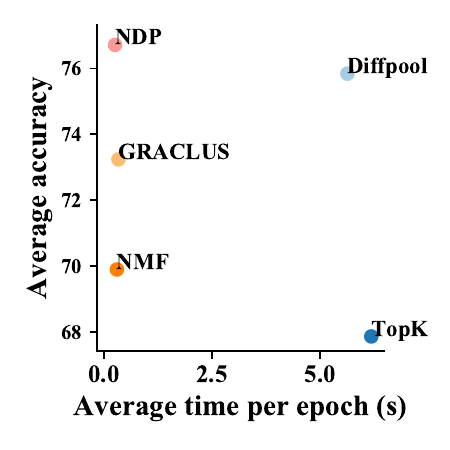}    
    \caption{Average training time per epoch against average accuracy, computed for each pooling method over the 10 graph classification tasks.}
    \label{fig:acc_vs_time}
\end{figure}

\begin{figure}[!ht]
    \centering
    \subfigure[Original graph.]{
        \includegraphics[width=0.5\columnwidth]{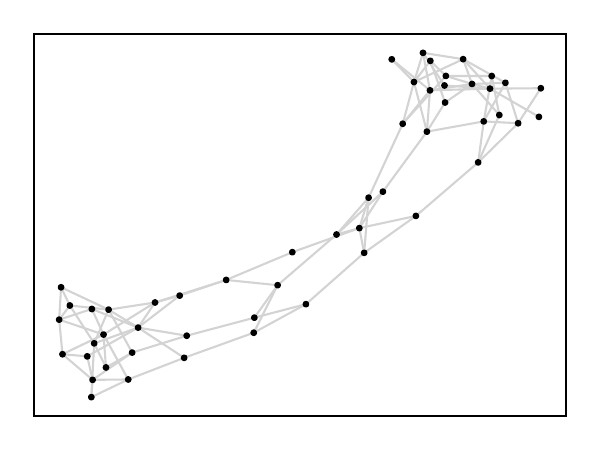}
    }\vspace{-.3cm}
    
    \subfigure[GRACLUS coarsening.]{
        \includegraphics[width=\columnwidth]{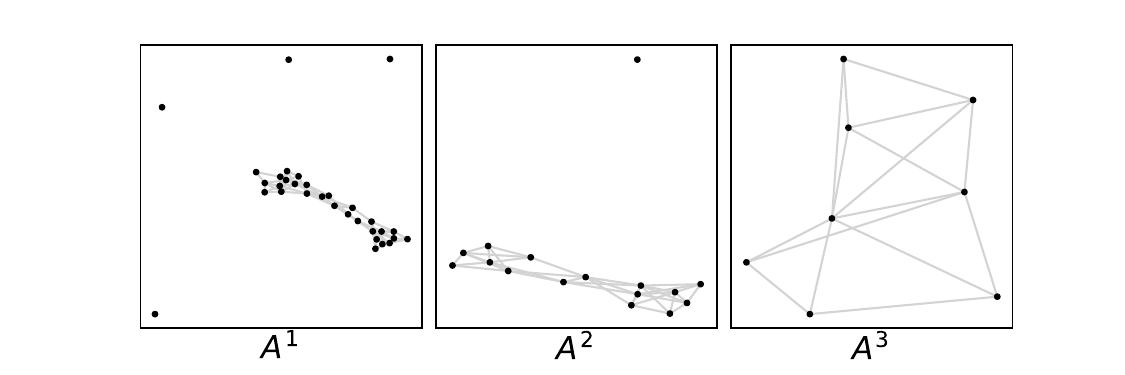}
    }\vspace{-.3cm}
    
    \subfigure[NMF coarsening.]{
        \includegraphics[width=\columnwidth]{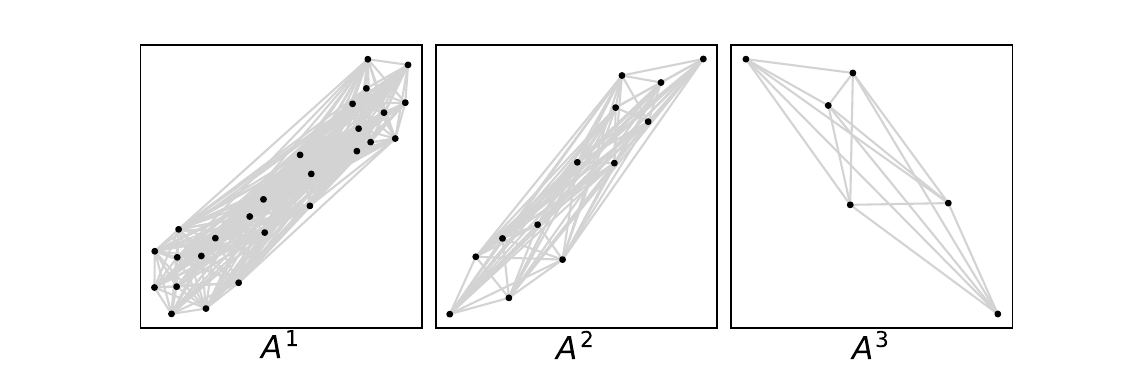}
    }\vspace{-.3cm}
    
    \subfigure[NDP coarsening.]{
        \includegraphics[width=\columnwidth]{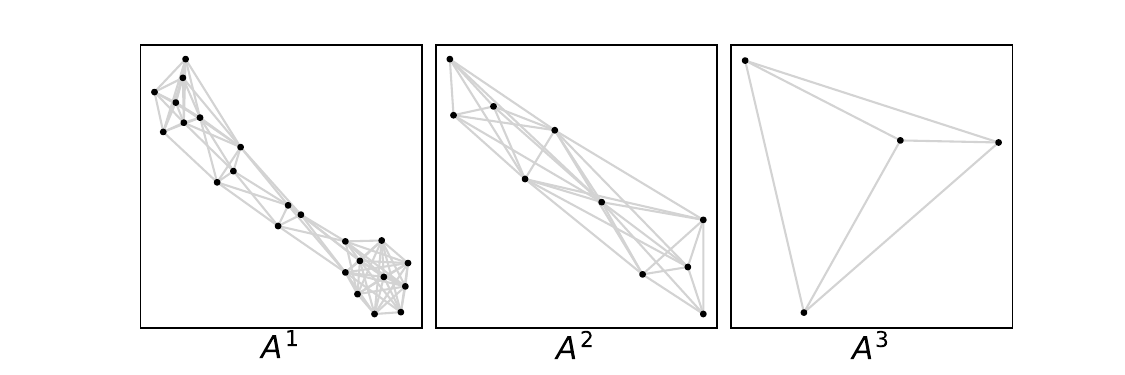}
    }
    \caption{\footnotesize Example of coarsening on one graph from the Proteins dataset. In (a), the original adjacency matrix of the graph. In (b), (c), and (d) the edges of the Laplacians at coarsening level 0, 1, and 2, as obtained by the 3 different pooling methods GRACLUS, NMF, and the proposed NDP.}
    \label{fig:coarsening_enzymes}
\end{figure}

To understand the differences between the topological pooling methods, we randomly selected one graph from the Proteins dataset and show in Fig.~\ref{fig:coarsening_enzymes} the coarsened graphs computed by GRACLUS, NMF, and NDP.
From Fig.~\ref{fig:coarsening_enzymes}(b) we notice that the graphs $\A^{(1)}$ and $\A^{(2)}$ in GRACLUS have additional nodes that are disconnected.
As discussed in Sect.~\ref{sec:related_work}, these are the fake nodes that are added to the graph so that its size can be halved at every pooling operation.
Fig.~\ref{fig:coarsening_enzymes}(c) shows that NMF produces graphs that are very dense, as a consequence of the multiplication with the dense soft-assignment matrix to construct the coarsened graph.
Finally, Fig.~\ref{fig:coarsening_enzymes}(d) shows that NDP produces coarsened graphs that are sparse and preserve well the topology of the original graph.

\subsection{Graph signal classification}
In this task, different graph signals $\mathbf{X}_i$, defined on the same adjacency matrix $\mathbf{A}$, must be classified with a label $\mathbf{y}_i$.
We use the same architecture adopted for graph classification, with the only difference that each pooling operation is now implemented with stride 4: MP(32)-P(4)-MP(32)-P(4)-MP(32)-AvgPool-Softmax.
We recall that when using NDP a stride of 4 is obtained by applying two decimation matrices in cascade, $\S^{(1)}\S^{(0)}$ and $\S^{(3)}\S^{(2)}$ (\textit{cf.} Sec.~\ref{sec:pooling_matrices}).
We perform two graph signal classification experiments: image classification on MNIST and sentiment analysis on IMDB dataset.
\newline

\textbf{MNIST.} 
For this experiment, we adopt the same settings described in \cite{defferrard2016convolutional}.
To emulate a typical convolutional network operating on a regular 2D grid, an 8-NN graph is defined on the $28 \times 28$ pixels of the MNIST images, using as edge weights the following similarity score between nodes:
\begin{equation}
\label{eq:edges}
    a_{ij} = \mathrm{exp}\left(- \frac{\| p_i - p_j \|^2}{\sigma^2}\right),
\end{equation}
where $p_i$ and $p_j$ are the 2D coordinates of pixel $i$ and $j$.
The graph signal $\X_i \in \mathbb{R}^{784 \times 1}$ is the $i$-th vectorized image.

\begin{table}
\setlength\tabcolsep{.7em} 
\small
\bgroup
\def\arraystretch{1.25} 
\centering
\caption{Graph signal classification accuracy on MNIST.}
\begin{tabular}{ccccc}
\cmidrule[1.5pt]{1-5}
\textbf{DiffPool} & \textbf{Top-$K$} & \textbf{GRACLUS} & \textbf{NMF} & \textbf{NDP} \\
\cmidrule[.5pt]{1-5}
24.00 {\tiny$\pm$ 0.0} & 11.00 {\tiny$\pm$ 0.0} & 96.21 {\tiny$\pm$ 0.18} & 94.15{\tiny$\pm$ 0.17} & \textbf{97.09 {\tiny$\pm$ 0.01}} \\
\cmidrule[1.5pt]{1-5}
\end{tabular}
\label{tab:mnist_res}
\egroup
\end{table}

Tab.~\ref{tab:mnist_res} reports the average results achieved over 10 independent runs by a GNN implemented with different pooling operators.
Contrarily to graph classification, DiffPool and Top$K$ fail to solve this task and achieve an accuracy comparable to random guessing. 
On the contrary, the topological pooling methods obtain an accuracy close to a classical CNN, with NDP significantly outperforming the other two techniques. 

We argue that the poor performance of the two feature-based pooling methods is attributable to 1) the low information content in the node features, and 2) a graph that has a regular structure and is connected only locally. 
This means that the graph has a very large diameter (maximum shortest path), where information propagates slowly through MP layers.
Therefore, even after MP, nodes in very different parts of the graph will end up having similar (if not identical) features, which leads feature-based pooling methods to assign them to the same cluster.
As a result the graph collapses, becoming densely connected and losing its original structure.
On the other hand, topological pooling methods can preserve the graph structure by operating on the whole adjacency matrix at once to compute the coarsened graphs and are not affected by uninformative node features.

\begin{table*}
\centering
\setlength\tabcolsep{1em} 
\small
\bgroup
\def\arraystretch{1.25} 
\caption{Graph signal classification accuracy on IMDB sentiment analysis dataset.}
\begin{tabular}{lcccccccc}
\cmidrule[1.5pt]{1-9}
\textbf{\# Words} & \textbf{Dense} & \textbf{LSTM} & \textbf{TCN} & \textbf{DiffPool} & \textbf{Top-$K$} & \textbf{GRACLUS} & \textbf{NMF} & \textbf{NDP} \\
\cmidrule[.5pt]{1-9}
1k  & 82.65{\tiny$\pm$0.01} & 86.58{\tiny$\pm$0.03} & 85.61{\tiny$\pm$0.14} & 50.00{\tiny$\pm$0.0} & 50.00{\tiny$\pm$0.0} & 85.03{\tiny$\pm$0.10} & 82.51{\tiny$\pm$0.11} & \textbf{85.77{\tiny$\pm$0.03}} \\ 
5k  & 86.26{\tiny$\pm$0.03} & 86.59{\tiny$\pm$0.06} & 87.42{\tiny$\pm$0.09} & 50.00{\tiny$\pm$0.0} & 50.00{\tiny$\pm$0.0} & 87.55{\tiny$\pm$0.15} & 85.66{\tiny$\pm$0.11} & \textbf{87.79{\tiny$\pm$0.02}}\\ 
10k & 83.75{\tiny$\pm$0.02} & 85.98{\tiny$\pm$0.04} & 87.38{\tiny$\pm$0.07} & 50.00{\tiny$\pm$0.0} & 50.00{\tiny$\pm$0.0} & 87.29{\tiny$\pm$0.07} & OOM & \textbf{87.82{\tiny$\pm$0.02}} \\ 
\cmidrule[1.5pt]{1-9}
\end{tabular}
\label{tab:imdb_res}
\egroup
\end{table*}

\textbf{IMDB.} 
We consider the IMDB sentiment analysis dataset of movies reviews, which must be classified as positive or negative.
We use a graph that encodes the similarity of all words in the vocabulary. 
Each graph signal represents a review and consists of a binary vector with size equal to the vocabulary, which assumes value 1 in correspondence of a word that appears at least once in the review, and 0 otherwise.

The graph is built as follows.
First, we extract a vocabulary from the most common words in the reviews.
For each review, we consider at most 256 words, padding with a special token the reviews that are shorter and truncating those that are longer.
Then, we train a simple classifier consisting of a word embedding layer~\cite{mikolov2013distributed} of size 200, followed by a dense layer with a ReLU activation, a dropout layer~\cite{srivastava2014dropout} with probability 0.5, and a dense layer with sigmoid activation. 
After training, we extract the embedding vector of each word in the vocabulary and construct a $4$-NN graph, according to the Euclidean similarity between the embedding vectors.

As baselines, we consider the network used to generate the word embeddings (\textit{Dense}) and two more advanced architectures.
The first (\textit{LSTM}), is a network where the dense hidden layer is replaced by an LSTM layer~\cite{hochreiter1997long}, which allows capturing the temporal dependencies in the sequence of words in the review.
The other baseline (\textit{TCN}) is a network where the hidden layers are 1D convolutions with different dilation rates~\cite{oord2016wavenet}.
In particular, we used a Temporal Convolution Network~\cite{bai2018empirical} with 7 residual blocks with dilations $[1, 2, 4, 8, 16, 32, 64]$, kernel size 6, causal padding, and dropout probability 0.3.
The results averaged over 10 runs for vocabularies of different sizes (\# Words) are reported in Tab.~\ref{tab:imdb_res}.

Similarly to the MNIST experiment, we notice that neither DiffPool nor Top$K$ are able to solve this graph signal classification task.
The reason can be once again attributed to the low information content of the individual node features and in the sparsity of the graph signal (most node features are 0), which makes it difficult for the feature-based pooling methods to infer global properties of the graph by looking at local sub-structures. 

On the other hand, NDP consistently outperforms the baselines, GRACLUS, and NMF. 
The coarsened graphs generated by NMF when the vocabulary has 10k words are too dense to fit in the memory of the GPU (Nvidia GeForce RTX 2080).
Interestingly, the GNNs configured with GRACLUS and NDP always achieve better results than the \textit{Dense} network, even if the latter generates the word embeddings used to build the graph on which the GNN operates. This can be explained by the fact that the \textit{Dense} network immediately overfits the dataset, whereas the graph structure provides a strong regularization, as the GNN combines only words that are neighboring on the vocabulary graph.  

The \textit{LSTM} baseline generally achieves a better accuracy than \textit{Dense}, since it captures the sequential ordering of the words in the reviews, which also helps to prevent overfitting on training data.
Finally, the \textit{TCN} baseline always outperforms \textit{LSTM}, both in terms of accuracy and computational costs. 
This substantiates recent findings showing that convolutional architectures may be more suitable than recurrent ones for tasks involving sequential data~\cite{bai2018empirical}.

\section{Conclusions}
\label{sec:conclusion}

We proposed Node Decimation Pooling (NDP), a pooling strategy for Graph Neural Networks that reduces a graph based on properties of its Laplacian. 
NDP partitions the nodes into two disjoint sets by optimizing a \maxcut{} objective. 
The nodes of one set are dropped, while the others are connected with Kron reduction to form a new smaller graph. 
Since Kron reduction yields dense graphs, a sparsification procedure is used to remove the weaker connections.

The algorithm we proposed to approximate the \maxcut{} solution is theoretically grounded on graph spectral theory and achieves good results while being, at the same time, simple and efficient to implement.
To evaluate the \maxcut{} solution, we considered theoretical bounds and we introduced an experimental framework to empirically assess the quality of the solution.

We demonstrated that the graph sparsification procedure proposed in this work preserves the spectrum of the graph up to an arbitrarily small constant.
In particular, we first derived an analytical relationship between the eigenvalues of the adjacency matrix of the original and sparsified graphs.
Then, we performed numerical experiments to study how much the spectrum of the graph Laplacian varies, in practice, after sparsification.

We compared NDP with two main families of pooling methods for GNNs: topological (to which NDP belongs) and feature-based methods.
NDP has advantages compared to both types of pooling. 
In particular, experimental results showed that NDP is computationally cheaper (in terms of both time and memory) than feature-based methods, while it achieves competitive performance on all the downstream tasks taken into account. 
An important finding in our results indicates that topological methods are the only viable approach in graph signal classification tasks.

\subsection*{Acknowledgments}
We are grateful to the anonymous reviewers for critically reading our manuscript and for giving us important suggestions, which allowed us to significantly improve our work.

\appendix

\subsection{Kron reduction in graph with self-loops}
\label{sec:kron_loops}
If $\A$ contains self loops, the existence of the strict inequality condition $\L_{ii} > \sum_{j=1, j \neq i}^n |\L_{ij}|$ discussed in Sec.~\ref{sec:links_construction} is no more guaranteed.
However, it is sufficient to consider the loopy-Laplacian $\Q = \mathbf{D} - \A + 2 \text{diag}(\A)$, where $\text{diag}(\A)$ is the diagonal of $\A$, defined as $\{\A_{ii}\}_{i=1}^N$.
$\Q$ is now an irreducible matrix and $\Q_{ii} > \sum_{j=1, j \neq i}^n |\Q_{ij}| + \A_{ii}$ holds for at least least one vertex $i \in \mathcal{V}^{+}$.
We notice that the adjacency matrix can be univocally recovered: $\A = -\Q + \text{diag}(\{ \sum_{j=1, j\neq i}^N \Q_{ij}\}_{i=1}^N)$.
Therefore, from the Kron reduction $\Q^{(1)}$ of $\Q$ we can first recover $\A^{(1)}$ and then compute the reduced Laplacian as $\L^{(1)} = \D^{(1)} - \A^{(1)}$.

\subsection{Derivation of the \maxcut{} upperbound}
\label{sec:upperbound_derivation}
Let us consider the Rayleigh quotient
\begin{equation}
    \label{eq:rayleigh}
    r(\z, \L) = \frac{\z^T \L \z}{\z^T\z},
\end{equation}
which assumes its maximum value $\lambda_\text{max}$ when $\z$ is the largest eigenvector of the Laplacian $\L$.
When $\z$ is the partition vector in \eqref{eq:partition_vec}, we have $r(\z, \L) \leq \lambda_\text{max}$.
As shown in Sect.~\ref{sec:maxcut}, the numerator in \eqref{eq:rayleigh} can be rewritten as 
$\z^T \L \z = \sum_{i,j \in \mathcal{E}} \frac{a_{ij}(z_i - z_j)^2}{2} = 
\sum_{i,j \in \mathcal{V}^{+}} \frac{a_{ij}(z_i - z_j)^2}{2} + \sum_{i,j \in \mathcal{V}^{-}} \frac{a_{ij}(z_i - z_j)^2}{2} + \sum_{i \in \mathcal{V}^{+}, j \in \mathcal{V}^{-}} \frac{a_{ij}(z_i - z_j)^2}{2} = 
0 + 0 + \sum_{i \in \mathcal{V}^{+}, j \in \mathcal{V}^{-}} \frac{a_{ij} 2^2}{2} = 
2\cdot\text{cut}(\z)$, since $z_i = 1$ if $i \in \mathcal{V}^{+}$ and $z_i = -1$ if $i \in \mathcal{V}^{-}$ according to \eqref{eq:partition_vec}, and where $\text{cut}(\z)$ is the volume of edges crossing the partition induced by $\z$.
From \eqref{eq:partition_vec} also follows that the denominator in \eqref{eq:rayleigh} is $\z^T\z = N$, since $z_i^2 = 1, \forall i$.
By combining the results we obtain
\begin{equation}
    \label{eq:N_bound}
    \frac{2\cdot\text{cut}(\z)}{N} \leq \lambda_\text{max}, \forall \z \in \mathbb{R}^N \rightarrow \maxcut{} \leq  \lambda_\text{max} \frac{N}{2}.
\end{equation}

When considering the symmetric Laplacian $\L_s$, we multiply \eqref{eq:rayleigh} on both sides by $\D^{-1/2}$, changing the denominator into $\z^T\D\z = \sum_{i,i}d_{ii} z_i^2 = |\mathcal{E}|$.
Replacing in \eqref{eq:N_bound} $N$ with $|\mathcal{E}|$ and $\lambda_\text{max}$ with $\lambda^s_\text{max}$, we get the bound $\maxcut{} / |\mathcal{E}| \leq \lambda^s_\text{max} / 2$.

\subsection{Relationship with Trevisan \cite{trevisan2012max} spectral algorithm}
\label{sec:trevisan}
The main result in \cite{trevisan2012max} states that if $\lambda^s_\text{max} \geq 2(1-\tau)$, then there exist a set of vertices $\mathcal{V}$ and a partition $(\mathcal{V}^1, \mathcal{V}^2)$ of $\mathcal{V}$ so that $|e(\mathcal{V}^1, \mathcal{V}^2)| \geq \frac{1}{2} (1 \sqrt{16 \tau})\text{vol}(\mathcal{V})$, where $\text{vol}(\mathcal{V}) = \sum_{i \in \mathcal{V}} d_i$ and $e(\mathcal{V}^1, \mathcal{V}^2)$ are the edges with one endpoint in $\mathcal{V}^1$ and the other in $\mathcal{V}^2$. 
In cases where an optimal solution cuts $1-\tau$ fraction of the edges, a partition found by a recursive spectral algorithm will remove $1 - 4\sqrt{\tau} +8\tau$ of the edges. 
The optimal $\tau$ is value 0.0549 for which $\frac{1 - 4\sqrt{\tau} +8\tau}{1-\tau}$ reaches its minimum 0.5311. 
When the largest eigenvalue $\lambda^s_\text{max}$ is too small, the expected random cut is larger than the solution found by the spectral algorithm.
The analysis in \cite{trevisan2012max} shows that the spectral cut is guaranteed to be larger than the random cut only when $\lambda^s_\text{max} \geq 2(1-\tau)$, \textit{i.e.}, when $\lambda^s_\text{max} \geq 1.891$ given the optimal value $\tau = 0.0549$.
Therefore, an algorithm that recursively cuts a fraction of edges according to the values in $\v^s_\text{max}$ until $\lambda^s_\text{max} \geq 2(1-\tau)$ and then performs a random cut, finds a solution that is always $\geq0.5311$ \maxcut{}. 

\bibliographystyle{IEEEtran}
\bibliography{references}

\clearpage

\onecolumn

\section*{Supplementary material}
\setcounter{subsection}{0}

\subsection{Cut size on regular and random graphs}

In the following, we consider two different types of bipartite graphs, a regular grid and a ring graph, and four classes of random graphs, which are the Stochastic Block Model (SBM), a sensor network, the Erdos-Renyi graph, and a community graph. 
A graphical representation in the node space and the adjacency matrix of one instance of each graph type is depicted in Fig.~\ref{fig:graph_types}.

\begin{figure*}[!ht]
    \centering
    \includegraphics[keepaspectratio, width=.75\textwidth]{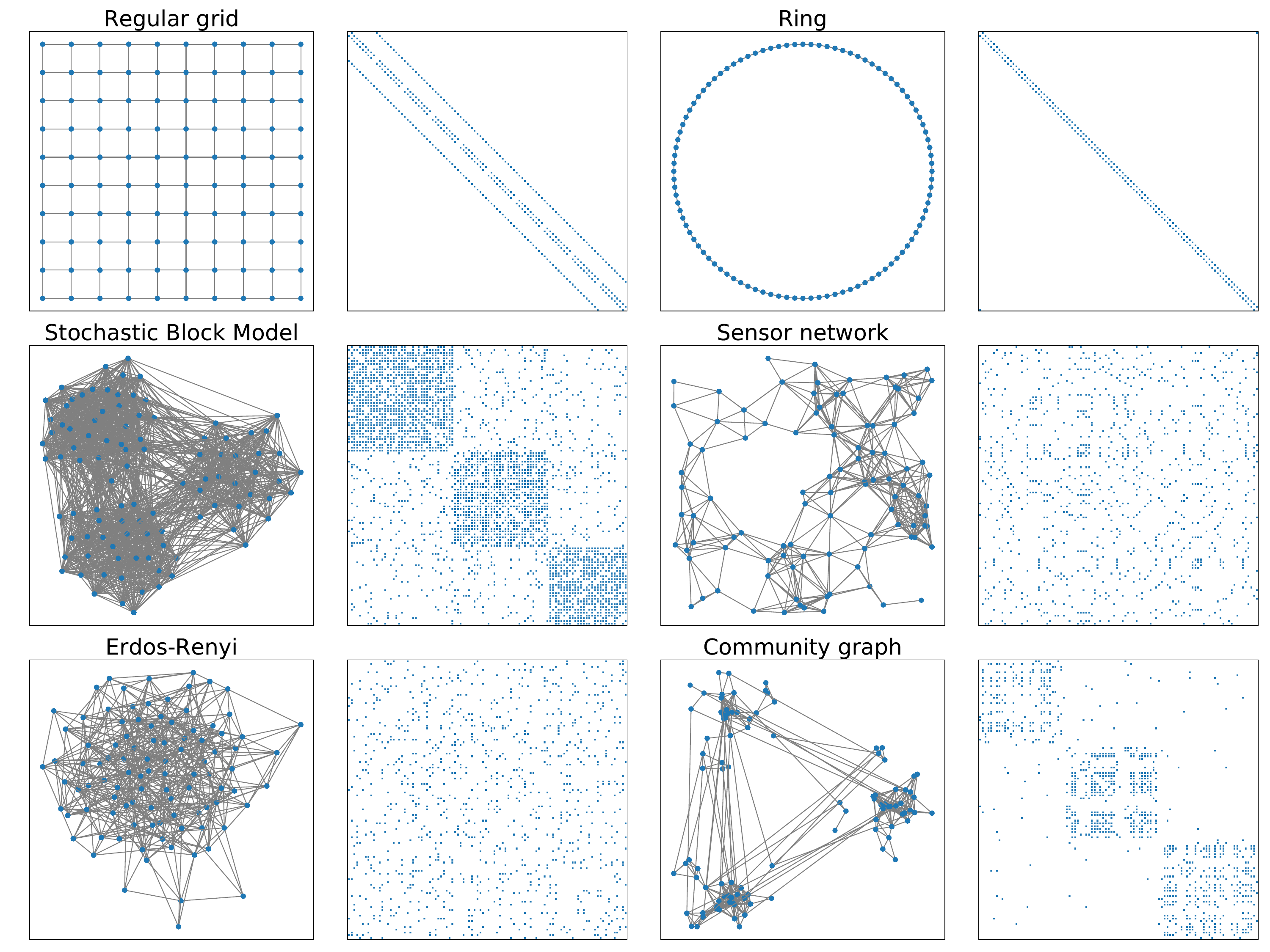} 
    \caption{Graphical representation and adjacency matrix of 2 regular graphs (\textit{Regular grid} and \textit{Ring}) and an instance of 4 random graphs (\textit{Stochastic Block Model (SBM)}, \textit{Sensor network}, \textit{Erdos-Renyi graph}, and \textit{Community graph}).}
    \label{fig:graph_types}
\end{figure*}

In Tab.~\ref{tab:cut_size} we report the size of the cut $\gamma(\mathbf{z})$ induced by the partition $\mathbf{z}$, which is obtained with the proposed spectral algorithm.
We consider both the partitions obtained from the eigenvectors $\v_\text{max}$ and $\v^s_\text{max}$, associated with the largest eigenvalue of the Laplacian $\mathbf{L}$ and the symmetric Laplacian $\mathbf{L}_s$, respectively.
The values in Tab.~\ref{tab:cut_size} are the mean and standard deviation of $\gamma(\mathbf{z})$ obtained on 50 different instances of each class.
We also report the \maxcut{} upperbound, $\lambda_\text{max}^\text{s}/2$ and the size of the cut induced by a random partition. 
\begin{table*}[!ht]
    \footnotesize
    \centering
    \caption{Size of the cut obtained with our spectral algorithm on different types of graph. Reported is the mean and standard deviation of the cut obtained from $\v_\text{max}$ and $\v^s_\text{max}$ on 50 instances of each graph type and the \maxcut{} upperbound, $\lambda_\text{max}^\text{s}/2$. For completeness, we show also the results obtained by the random cut.} 
    \label{tab:cut_size}
    \begin{tabular}{lcccccc}
    \cmidrule[1.5pt]{1-7}
                                & \textbf{Grid} & \textbf{Ring} & \textbf{SBM}      & \textbf{Sensor}   & \textbf{Erdos-Renyi}  & \textbf{Community} \\
    \cmidrule[.5pt]{1-7}
    \maxcut{} upperbound        & 1.0           & 1.0           & 0.63$\pm$0.0      & 0.77$\pm$0.05     & 0.67$\pm$0.0          & 0.89$\pm$0.06 \\
    Cut with $\v_\text{max}$    & 1.0           & 1.0           & 0.51$\pm$0.03     & 0.53$\pm$0.03     & 0.55$\pm$0.02         & 0.5$\pm$0.05 \\
    Cut with $\v^s_\text{max}$  & 1.0           & 1.0           & 0.58$\pm$0.01     & 0.58$\pm$0.02     & 0.61$\pm$0.0          & 0.54$\pm$0.04 \\
    Random cut                  & 0.5$\pm$0.03  & 0.5$\pm$0.05  & 0.5$\pm$0.01      & 0.51$\pm$0.02     & 0.5$\pm$0.0           & 0.5$\pm$0.01 \\
    \cmidrule[1.5pt]{1-7}
    \end{tabular}
\end{table*}
Consistently better performance are obtained when the partition is based on $\v^s_\text{max}$ rather than $\v_\text{max}$; as discussed in Sect. IV-A, this is because many entries in $\v_\text{max}$ have small values that cannot be partitioned precisely according to the sign, due to numerical errors.
The results show that on the two regular graphs, which are bipartite, the cut obtained with the spectral algorithm coincides with the \maxcut{} upper bound and, therefore, also with the optimal solution.  
For every other graph, the cut yielded by the spectral algorithm is always larger than the random cut. 
We recall that in those cases the \maxcut{} is unknown and the gaps between the lower bound (0.5) and the upper bound ($\lambda^s_\text{max}/2$) can be arbitrarily large.

\subsection{Spectral and random cut as a function of edge density}

\begin{figure}[!ht]
    \centering
    
    \subfigure[\footnotesize Regular grid]{
    \includegraphics[keepaspectratio,width=0.3\textwidth]{figs_r1/grid.pdf}
    }
    \subfigure[\footnotesize Ring graph]{
    \includegraphics[keepaspectratio,width=0.3\textwidth]{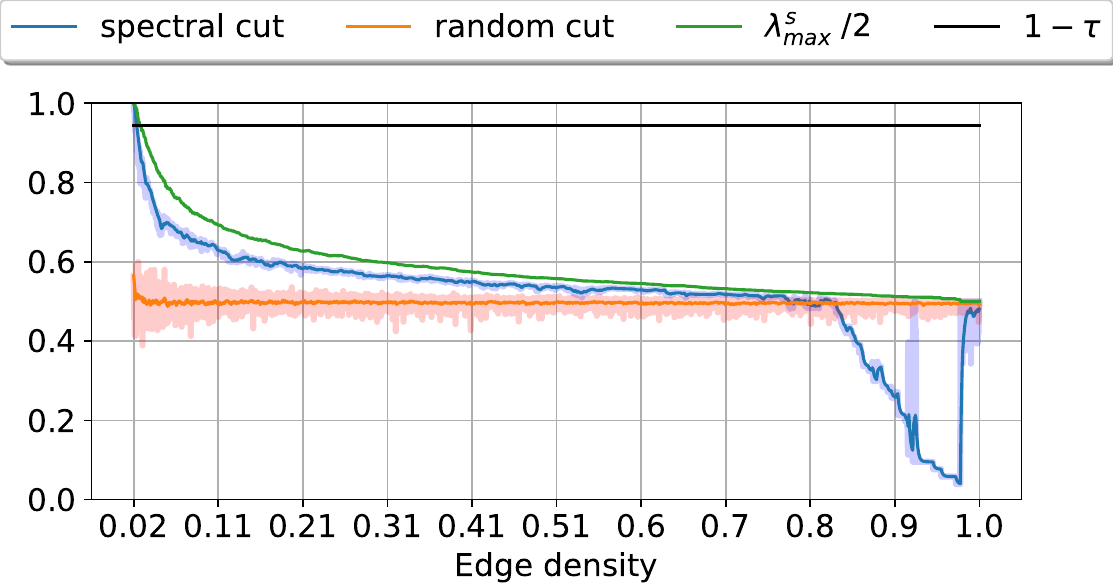}
    }
    \subfigure[\footnotesize Stochastic Block Model]{
    \includegraphics[keepaspectratio,width=0.3\textwidth]{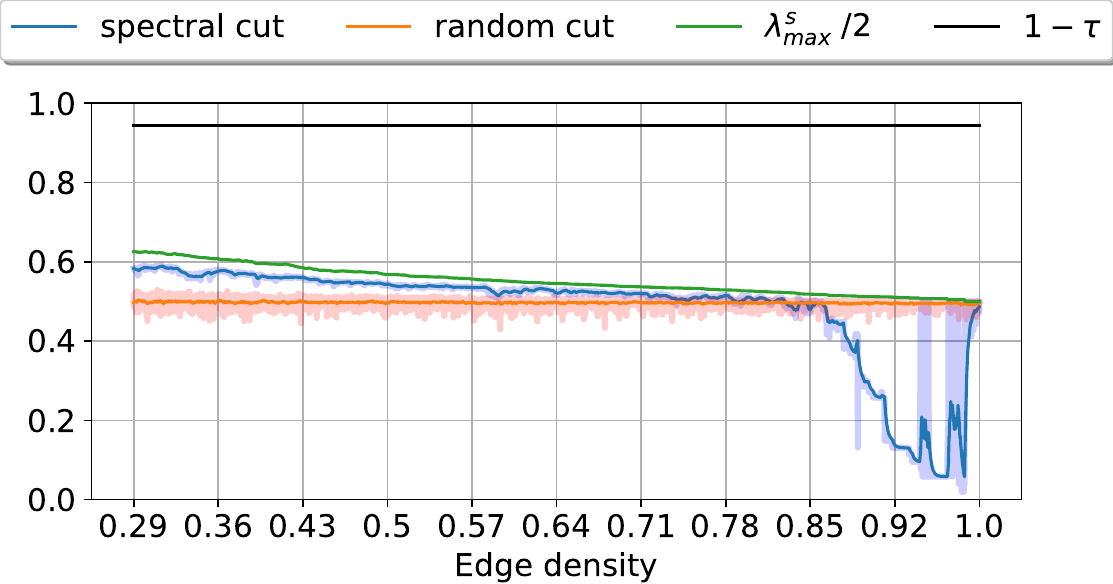}
    }

    \subfigure[\footnotesize Erdos-Renyi]{
    \includegraphics[keepaspectratio,width=0.3\textwidth]{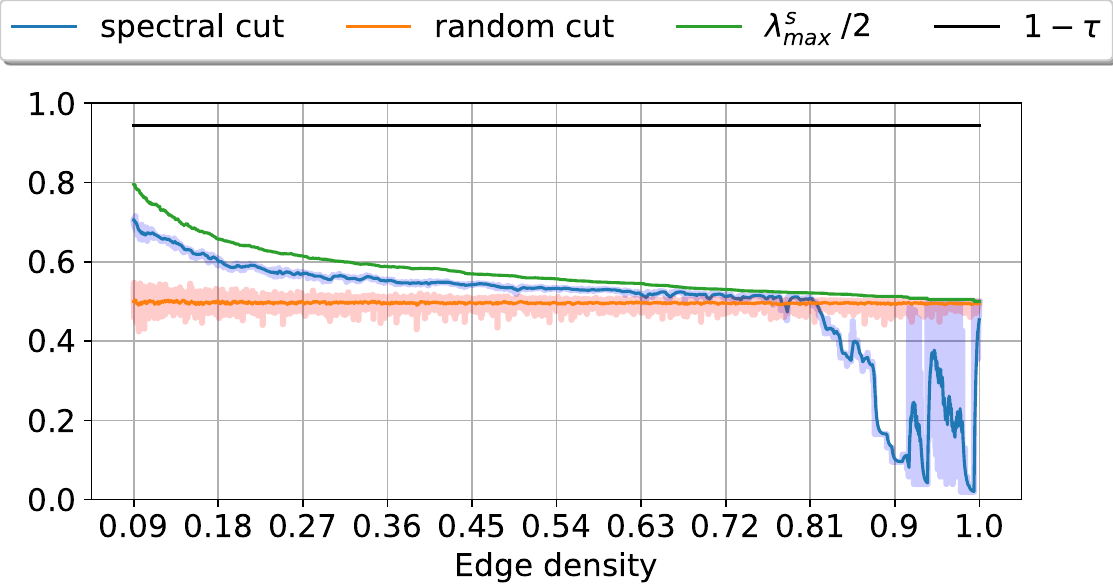}
    }
    \subfigure[\footnotesize Sensor network]{
    \includegraphics[keepaspectratio,width=0.3\textwidth]{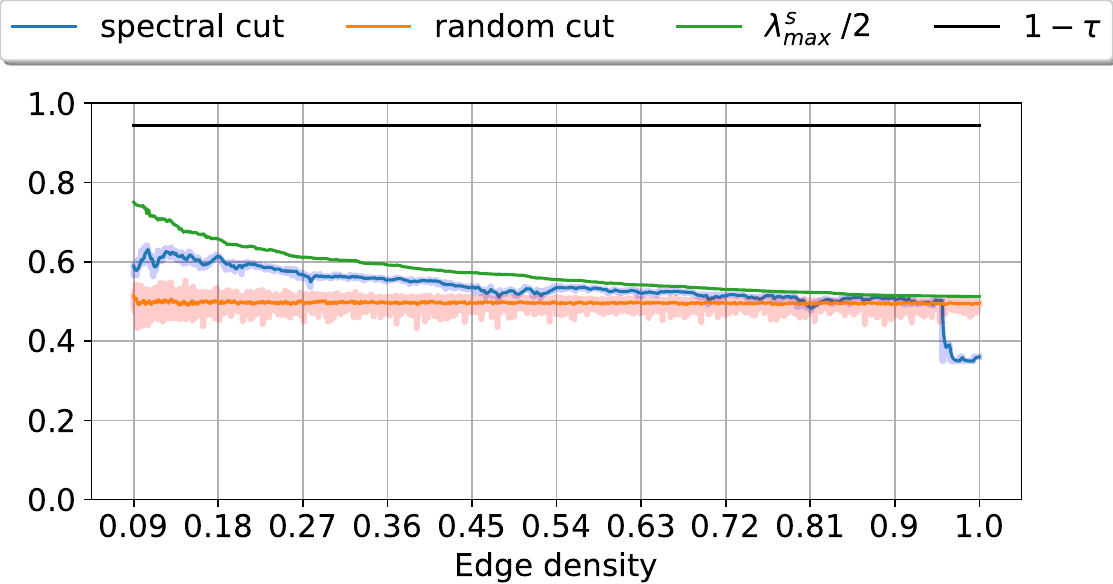}
    }
    \subfigure[\footnotesize Community graph]{
    \includegraphics[keepaspectratio,width=0.3\textwidth]{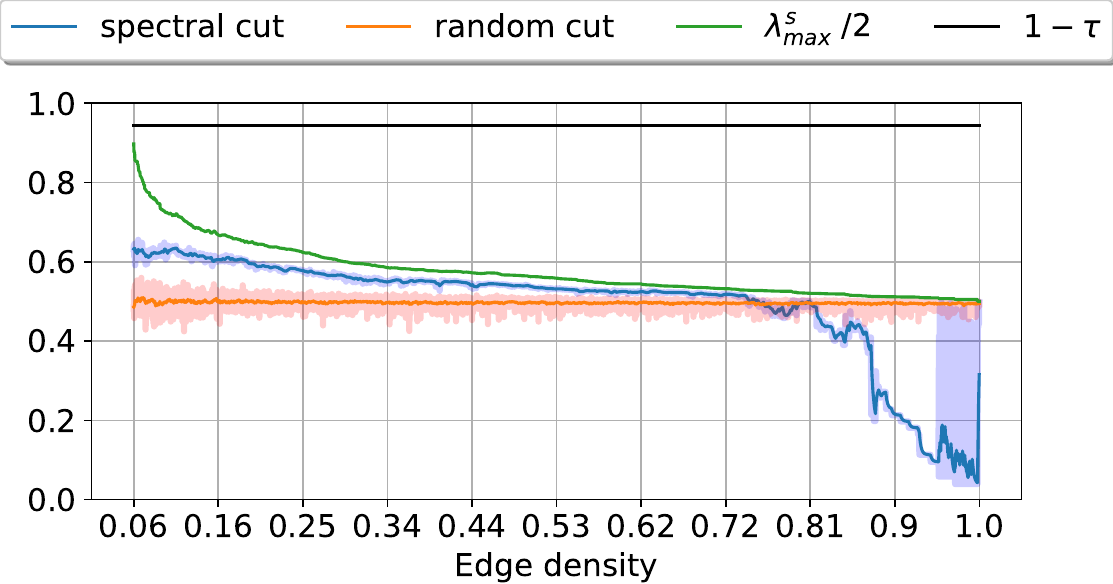}
    }
    \caption{Blue line: fraction of edges cut by the partition yielded by the spectral algorithm. 
    Orange line: fraction of edges removed by a random cut. 
    Green line: the \maxcut{} upper bound $\lambda^s_\text{max}/2$. 
    Black line: the threshold from \cite{trevisan2012max} indicating the value of $\lambda^s_\text{max}/2$ below which one should switch to the random cut to obtain a solution $\geq 0.53$ \maxcut{}. 
    The x-axis indicates the density of the graph connectivity, which increases by randomly adding edges.}
    \label{fig:adding_edges2}
\end{figure}

We replicate for each graph type the experiment in Sect. IV-B, which illustrates how the size of the cut obtained with the proposed algorithm changes as we randomly add edges.
Fig.~\ref{fig:adding_edges2} reports in blue the size of the cut associated with the partition yielded by the spectral algorithm; in orange the size of the cut yielded by the random partition; in green the \maxcut{} upper bound; in black the theoretical threshold that indicates when to switch to the random partition to obtain a cut with size $\geq 0.53$ \maxcut{}.

The examples encompass the two extreme cases where the \maxcut{} solution is known: a bipartite graph where \maxcut{} is 1 and the complete graph where \maxcut{} is 0.5.
In every example, when $\lambda^s_\text{max}$ becomes lower than $1-\tau$ the solution of the spectral algorithm is still larger than the cut induced by the random partition.
In fact, the spectral cut remains larger than the random cut until when the density is approximately 70-80\%.
Importantly, when the solution of the spectral algorithm become worse than the random cut, the \maxcut{} upper bound is close to 0.5.
Therefore, when the spectral cut is lower than 0.5 it is possible to return the random partition instead, which yields a nearly-optimal solution.

\subsection{Visual examples of coarsening with NDP pooling}

\begin{figure}[!pt]
    \centering
    \subfigure[\footnotesize Regular grid]{
    \includegraphics[keepaspectratio,width=0.47\textwidth]{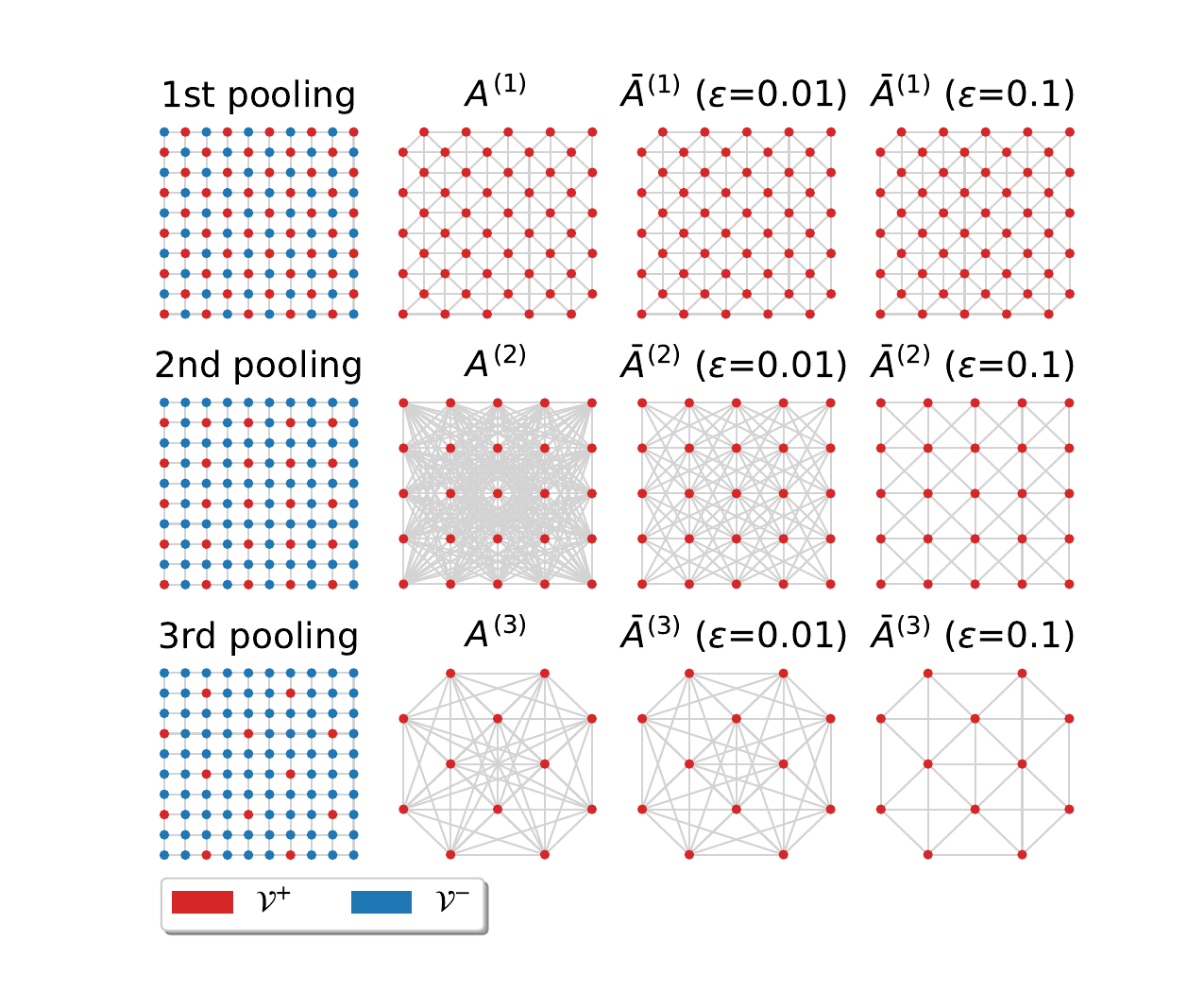}
    }
    \subfigure[\footnotesize Ring graph]{
    \includegraphics[keepaspectratio,width=0.47\textwidth]{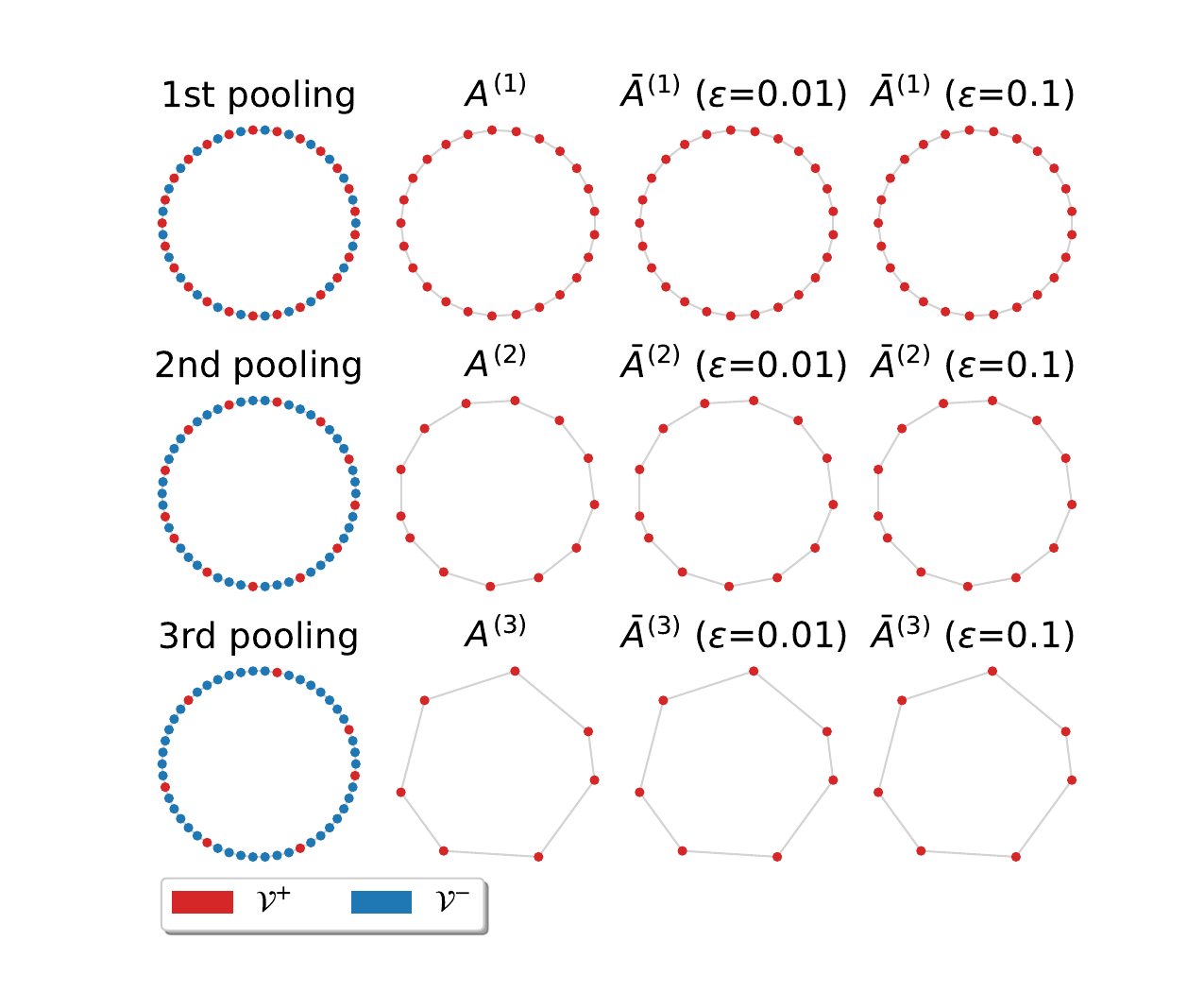}
    }
    
    \subfigure[\footnotesize Stochastic Block Model]{
    \includegraphics[keepaspectratio,width=0.47\textwidth]{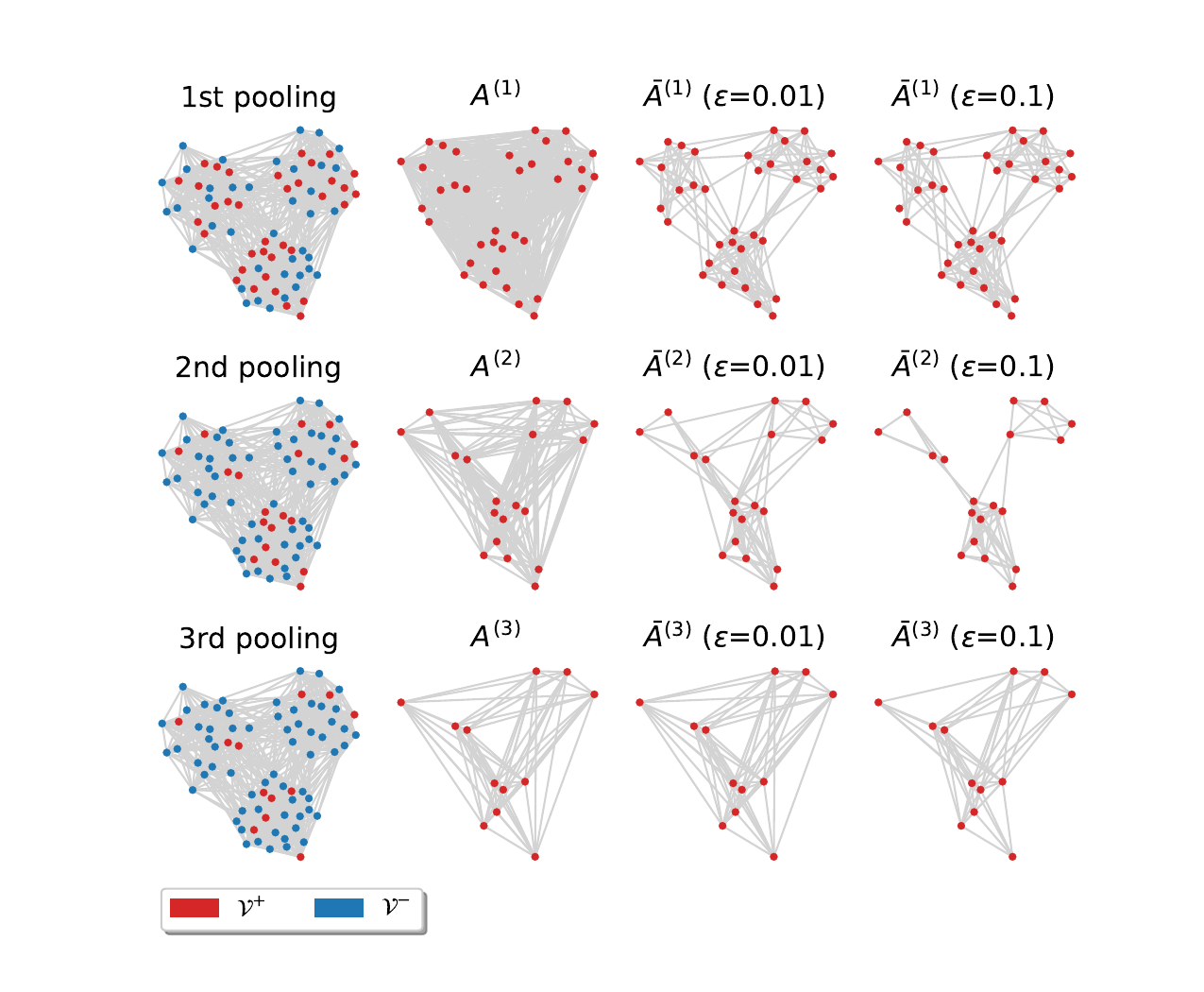}
    }
    \subfigure[\footnotesize Sensor network]{
    \includegraphics[keepaspectratio,width=0.47\textwidth]{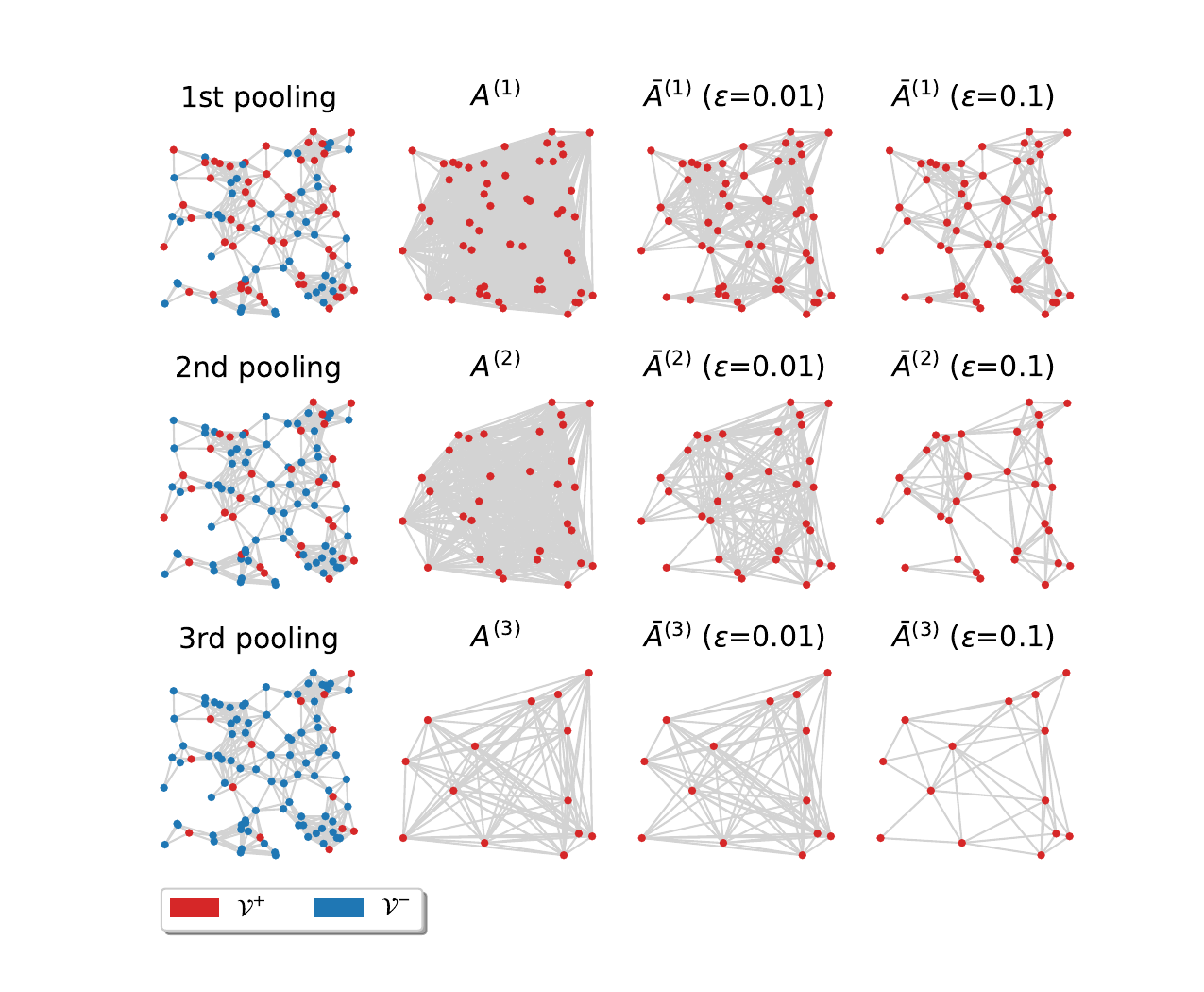}
    }
    
    \subfigure[\footnotesize Erdos-Renyi]{
    \includegraphics[keepaspectratio,width=0.47\textwidth]{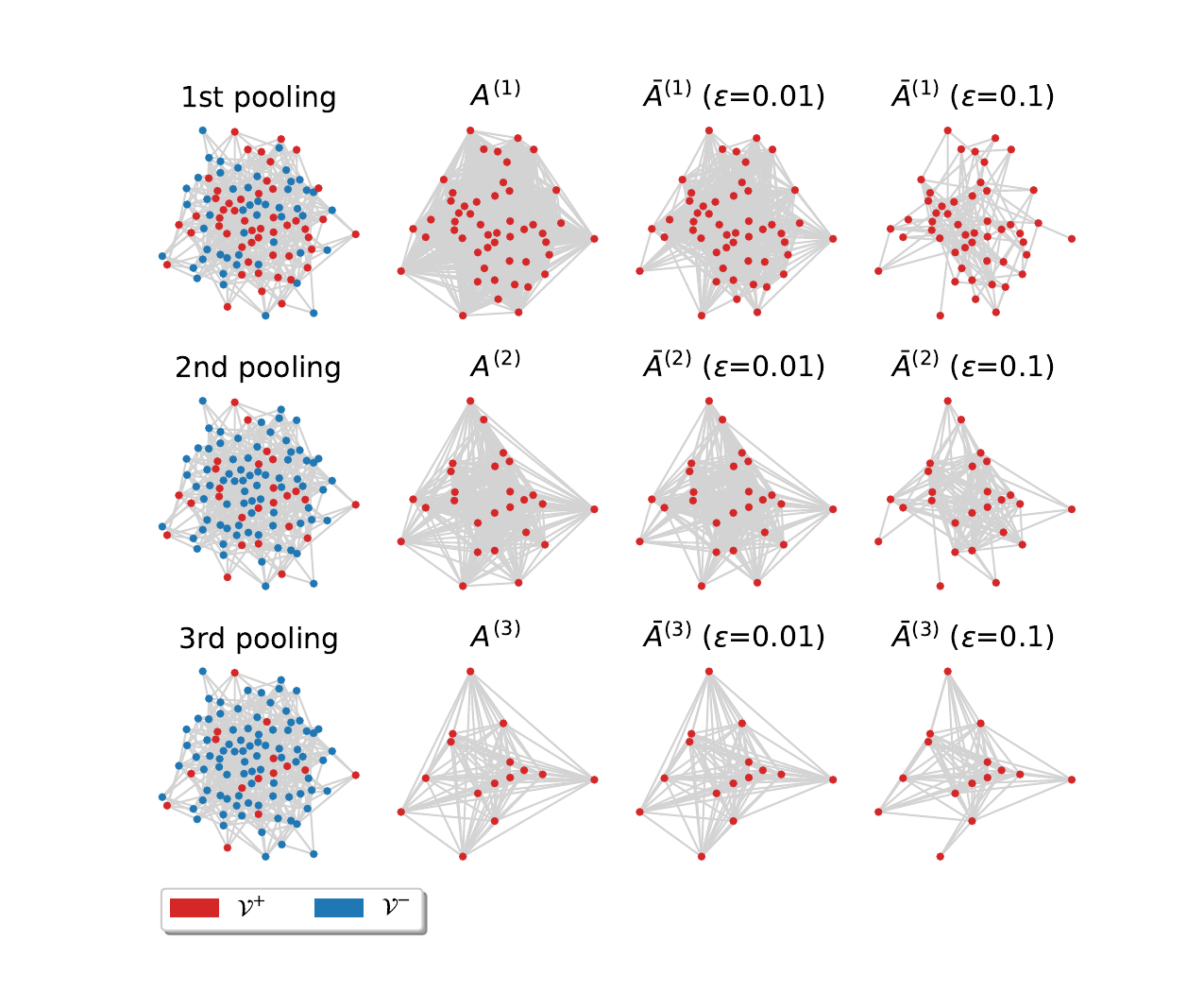}
    }
    \subfigure[\footnotesize Community graph]{
    \includegraphics[keepaspectratio,width=0.47\textwidth]{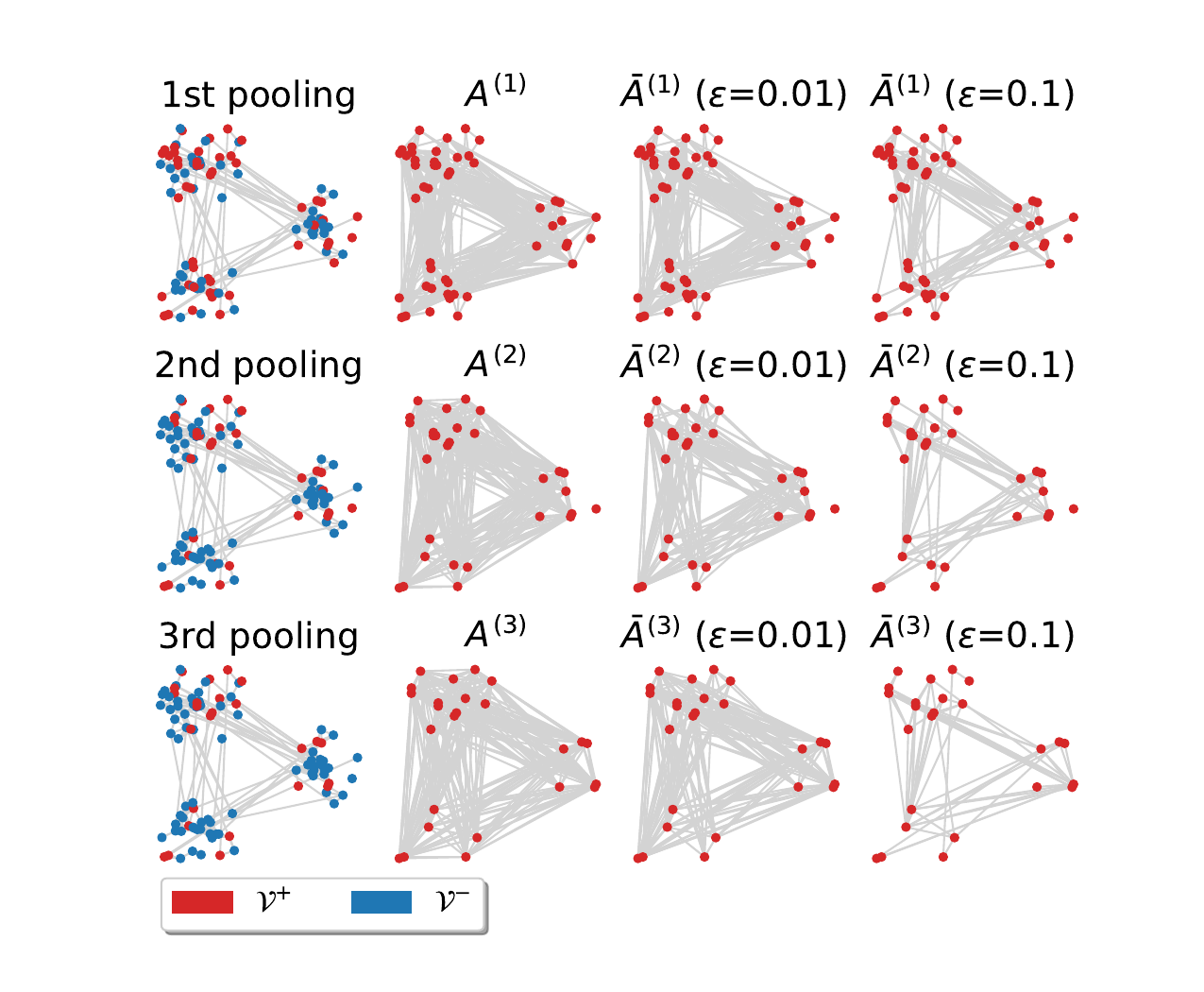}
    }
    
    \caption{Coarsened graphs obtained with the NDP algorithm. The 3rd and 4th column show graphs sparsified with different threshold $\epsilon$.}
    \label{fig:visualization_grid}
\end{figure}

Fig.~\ref{fig:visualization_grid} shows for the result of the NDP coarsening procedure on the 6 types of graphs.
The first column shows the subset of nodes of the original graph that are selected ($\mathcal{V}^{+}$, in red) and discarded ($\mathcal{V}^{-}$, in blue) after each pooling step.
The second column shows the coarsened graph obtained after each pooling operation.
Finally, columns 3 and 4 show the coarsened graphs after applying sparsification with different thresholds $\epsilon$.

\subsection{Spectral similarity in sparsified graphs}

In Sec. IV-E we introduced the spectral similarity distance to quantify how much the spectrum of the Laplacian associated with the sparsified adjacency matrix changes when edges smaller than $\epsilon$ are dropped.
In Fig.~\ref{fig:varying_eps2} we show how the graph structure (in terms of spectral similarity) varies, when the value of $\epsilon$ increases and more edges are dropped.
In every example, for small values of $\epsilon$ the structure of the graphs changes only slightly while a large amount of edges is dropped.
Notably, the spectral similarity increases almost linearly with $\epsilon$, while the edge density decreases exponentially.

\begin{figure}
    \centering
    \subfigure[\footnotesize Regular grid]{
    \includegraphics[keepaspectratio,width=0.3\textwidth]{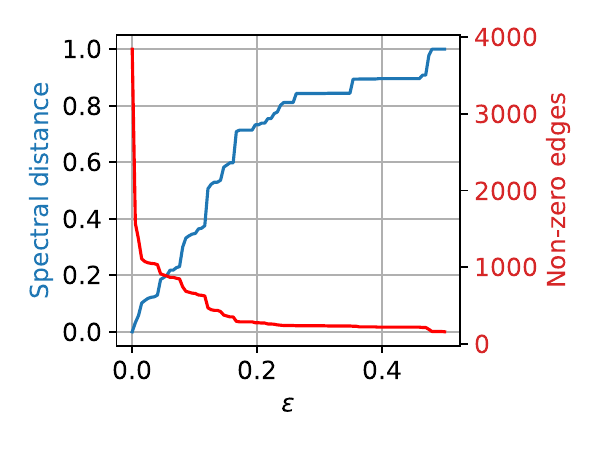}
    }
    ~
    \subfigure[\footnotesize Community graph]{
    \includegraphics[keepaspectratio,width=0.3\textwidth]{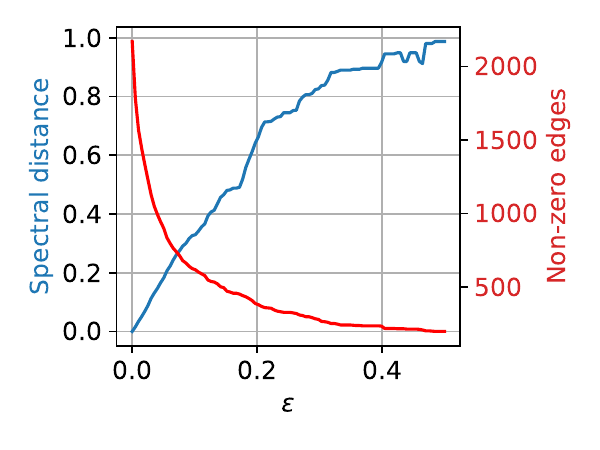}
    }
    ~
    \subfigure[\footnotesize Sensor network graph]{
    \includegraphics[keepaspectratio,width=0.3\textwidth]{figs_r1/sensor_eps.pdf}
    }
    \caption{In blue, the variation of spectral distance between the Laplacian $\L$ associated with $\A$ and the Laplacian $\bar{\L}$ associated with the adjacency matrix $\bar \A$ sparsified with a varying threshold $\epsilon$. In red, the number of edges that remain in $\bar{\L}$.}
    \label{fig:varying_eps2}
\end{figure}

\subsection{Mini-batch training}
Problems such as graph classification and graph regression are characterized by samples of graphs that, generally, have a variable number of vertices.
In order to apply MP and pooling operations when training a GNN on mini-batches, one solution is to perform zero-padding and obtain all graphs with $N_\text{max}$ vertices, where $N_\text{max}$ is the number of vertices in the largest graph of the dataset.
However, this solution is particularly inefficient in terms of memory cost, especially when there are many graphs with less than $N_\text{max}$ vertices.
A more efficient solution is to build the disjoint union of the graphs in each mini-batch and train the GNN on the combined Laplacian and graph signal.
This is the solution adopted in our experiments; Fig.~\ref{fig:graph_class} reports a visualization of the procedure.

\begin{figure*}[!ht]
    \centering
    \includegraphics[keepaspectratio, width=.75\textwidth]{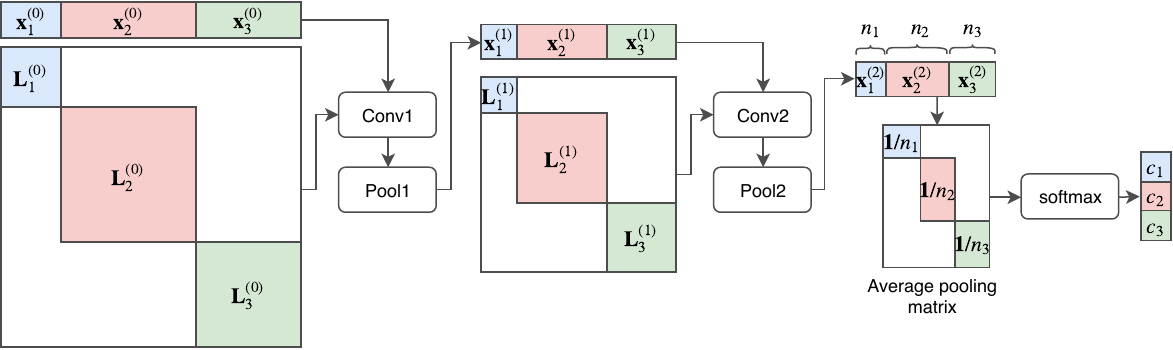}    
    \caption{Example of the implementation used in the graph classification task, where the GNN is fed with a disjoint union of the graphs in mini-batch. The illustration shows an example for a mini-batch of size three.}
    \label{fig:graph_class}
\end{figure*}

\subsection{Training curves}
Fig.~\ref{fig:learning} reports the evolution of the loss during training for 4 different graph classification datasets.
Notice that in our experiments we used early stopping. 
However, to provide a more extended profile of the training procedure, we show the training curves obtained when the GNN is trained for 1000 epochs on different datasets.

\begin{figure}[!ht]
    \centering
    \subfigure[Bench Hard]{\includegraphics[keepaspectratio,width=0.23\textwidth]{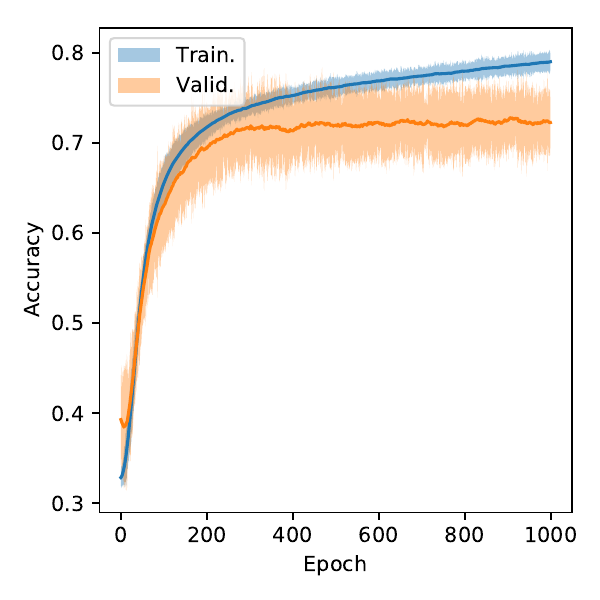}}
    ~
    \subfigure[NCI1]{\includegraphics[keepaspectratio,width=0.23\textwidth]{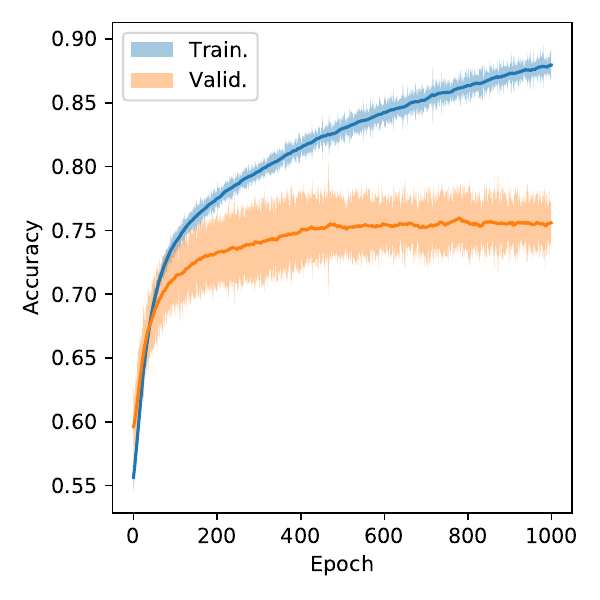}}
    ~
    \subfigure[MUTAG]{\includegraphics[keepaspectratio,width=0.23\textwidth]{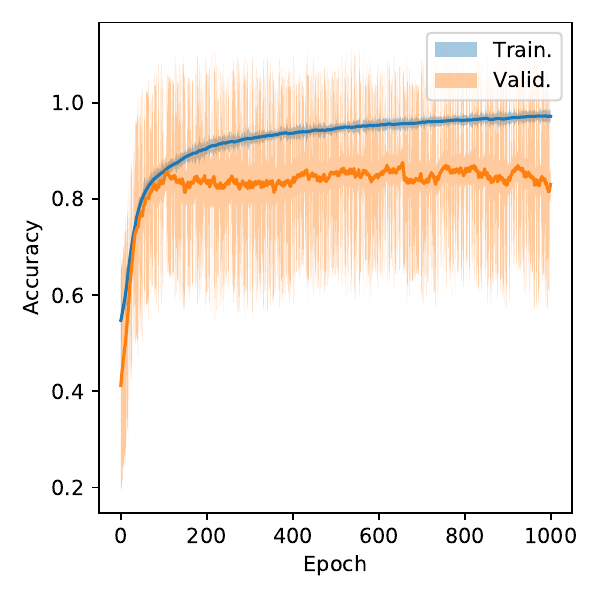}}
    ~
    \subfigure[Mutagenicity]{\includegraphics[keepaspectratio,width=0.23\textwidth]{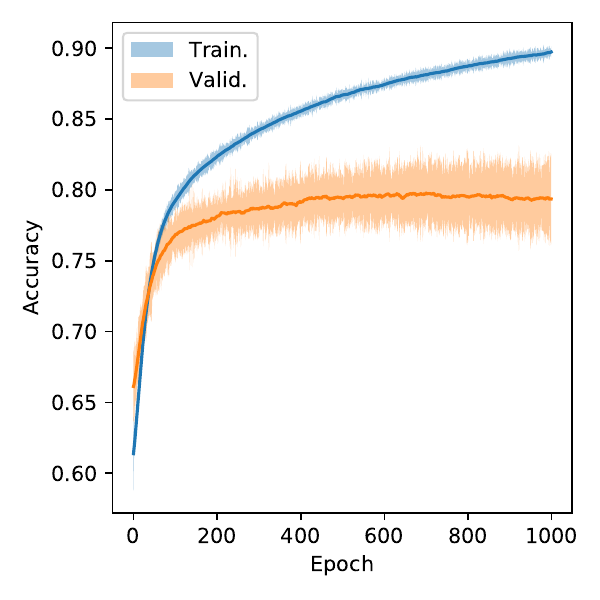}}
    
    \caption{Accuracy in training and validation over 1000 epochs on 4 different datasets. The curves are averaged over 10 runs per method and per dataset.}
    \label{fig:learning}
\end{figure}

\end{document}